\documentclass{article}

\usepackage{arxiv}
\usepackage{natbib}
\usepackage{caption}
\usepackage[utf8]{inputenc} 
\usepackage[T1]{fontenc}    
\usepackage{hyperref}       
\usepackage{url}            
\usepackage{booktabs}       
\usepackage{amsfonts}       
\usepackage{nicefrac}       
\usepackage{microtype}      
\usepackage{lipsum}
\usepackage{multirow}
\usepackage{graphicx}
\graphicspath{ {./images/} }
\usepackage{amssymb}
\usepackage{amsthm}
\usepackage{amsfonts}
\usepackage{verbatim}
\usepackage{xcolor}
\usepackage{array}
\usepackage{tabularx}
\usepackage{amsmath}

\newtheorem{definition}{Definition}[section]

\newtheorem{lemma}{Lemma}[section]
\newtheorem{theorem}{Theorem}

\newtheorem{prop}{Proposition}[section]

\newtheorem{claim}{Claim}[section]

\newcommand{\R}{\mathbb{R}}

\title{Sparse Linear Networks with a Fixed Butterfly Structure: Theory and Practice}

\begin{document}

\author{
\resizebox{0.9\textwidth}{!}{
\begin{tabular}{ccc}
   Nir Ailon & Omer Leibovitch & Vineet Nair  \\
   Faculty of Computer Science & Faculty of Computer Science & Faculty of Computer Science\\
   Technion Israel Institute of Technology &  Technion Israel Institute of Technology & Technion Israel Institute of Technology\\
   \texttt{nailon@cs.technion.ac.il} & \texttt{leibovitch@campus.technion.ac.il} & \texttt{vineet@cs.technion.ac.il} \\
\end{tabular}
}
}

\maketitle
\begin{abstract}
A butterfly network consists of logarithmically many layers, each with a linear number of non-zero weights (pre-specified). The fast Johnson-Lindenstrauss transform (FJLT) can be represented as a butterfly network followed by a  projection onto a random subset of the coordinates. Moreover, a random matrix based on FJLT with high probability approximates the action of any matrix on a vector. Motivated by these facts, we propose to replace a dense linear layer in any neural network by an architecture based on the butterfly network. The proposed architecture significantly improves upon the quadratic number of weights required in a standard dense layer to nearly linear with little compromise in expressibility of the resulting operator. In a collection of wide variety of experiments, including supervised prediction on both the NLP and vision data, we show that this not only produces results that match and at times outperform existing well-known architectures, but it also offers faster training and prediction in deployment. To understand the optimization problems posed by neural networks with a butterfly network, we also study the optimization landscape of the encoder-decoder network, where the encoder is replaced by a butterfly network followed by a dense linear layer in smaller dimension. Theoretical result presented in the paper explains why the training speed and outcome are not compromised by our proposed approach.
\end{abstract}








\section{Introduction}\label{SEC:INTRO}
A butterfly network (see Figure~\ref{figbfnet} in Appendix \ref{subsection bf plots}) is a layered graph connecting a layer of $n$ inputs to a layer of $n$ outputs with $O(\log n)$ layers, where each layer contains
$2n$ edges.  The edges connecting adjacent layers are organized in disjoint gadgets, each gadget connecting a pair of nodes in one layer with a corresponding pair in the next layer by a complete graph.  The distance between pairs doubles from layer to layer.  This network structure represents the execution graph of the Fast Fourier Transform (FFT) \citep{CooleyT65},  Walsh-Hadamard transform,  and many important transforms in signal processing that are known to have fast algorithms to compute matrix-vector products.

\citet{AilonC09} showed how to use the Fourier (or Hadamard) transform to perform fast Euclidean dimensionality reduction with  \cite{JohnsonL84} guarantees.  
The resulting transformation, called Fast Johnson Lindenstrauss Transform (FJLT), was  improved in subsequent work \citep{DBLP:journals/dcg/AilonL09,KrahmerW2011}.  
The common theme in this line of work is to define a fast randomized linear transformation that is composed of a random diagonal matrix, followed by a dense orthogonal transformation which can be represented via a butterfly network, followed by a random projection onto a subset of the coordinates (this research is still active, see e.g. \cite{kac}). In particular, an FJLT matrix can be represented (explicitly) by a butterfly network followed by projection onto a random subset of coordinates (a truncation operator). We refer to such a representation as a truncated butterfly network (see Section \ref{sec: prelim}).  

Simple Johnson-Lindenstrauss like arguments show that with high probability for any $W \in \R^{n_2\times n_1}$ and any $\mathbf{x}\in \R^{n_1}$, $W\mathbf{x}$ is close to $(J_2^TJ_2)W(J_1^TJ_1)\mathbf{x}$ where $J_1 \in \R^{k_1 \times n_1}$ and $J_2 \in \R^{k_2\times n_2}$ are both FJLT, and $k_1 = \log n_1, k_2 = \log n_2$ (see Section \ref{subsec: matrix approx using butterfly network} for details). Motivated by this, we propose to replace a dense (fully-connected) linear layer of size $n_2 \times n_1$ in any neural network by the following architecture: $J_1^TW'J_2$, where $J_1, J_2$ can be represented by a truncated butterfly network and $W'$ is a $k_2\times k_1$ dense linear layer. The clear advantages of such a strategy are: (1) almost all choices of the weights from a specific distribution, namely the one mimicking FJLT, preserve accuracy while reducing the number of parameters, and (2) the number of weights is nearly linear in the layer width of $W$ (the original matrix). Our empirical results demonstrate that this offers faster training and prediction in deployment while producing results that match and often outperform existing known architectures. Compressing neural networks by replacing linear layers with structured linear transforms that are expressed by fewer parameters have been studied extensively in the recent past. We compare our approach with these related papers in Section \ref{sec: related work}. \\

Since the butterfly structure adds logarithmic depth to the architecture, it might pose optimization related issues. Moreover, the sparse structure of the matrices connecting the layers in a butterfly network defies the general theoretical analysis of convergence of deep linear networks. We take a small step towards understanding these issues by studying the optimization landscape of an encoder-decoder network (two layer linear neural network), where the encoder layer is replaced by a truncated butterfly network followed by a dense linear layer in fewer parameters. This replacement is motivated by the result of \cite{Sarlos06}, related to fast randomized low-rank approximation of matrices using FJLT (see Section \ref{subsec: matrix approx using butterfly network} for details).\footnote{We could also have replaced the encoder matrix with the proposed architecture in Section \ref{subsec: matrix approx using butterfly network}, but in order to study the optimization issues posed by the truncated butterfly network we chose to study this simpler replacement. Moreover, even in this case the new network after replacing the encoder has very little loss in representation compared to the encoder-decoder network \cite{Sarlos06}.}\\

The encoder-decoder network computes the best low-rank approximation of the input matrix. It is well-known that with high probability \emph{a close to optimal} low-rank approximation of a matrix is obtained by either pre-processing the matrix with an FJLT \citep{Sarlos06} or a random sparse matrix structured as given in \citet{ClarksonW09}, and then computing the best low-rank approximation from the rows of the resulting matrix.\footnote{The pre-processing matrix is multiplied from the left.} A recent work by \citet{IndykVY19} studies this problem in the supervised setting, where they find the best pre-processing matrix structured as given in \cite{ClarksonW09} from a sample of matrices (instead of using a random sparse matrix). Since an FJLT can be represented by a truncated butterfly network, we emulate the setting of \citet{IndykVY19} but learn the pre-processing matrix structured as a truncated butterfly network. 
\subsection{Our Contribution and Potential Impact}
We provide a theoretical analysis together with an empirical report to justify our main idea of using sparse linear layers with a fixed butterfly network in deep learning.
Our findings indicate that this approach, which is well rooted in the theory of matrix approximation and optimization, can offer significant speedup and energy saving in deep learning applications.  Additionally, we believe that this work would encourage more experiments and theoretical analysis to better understand the optimization and generalization of our proposed architecture (see Section \ref{sec: conclusion}).  

{\bf On the theoretical side} -- 
The optimization landscape of linear neural networks with dense matrices have been studied by \citet{BaldiH89}, and \citet{Kawaguchi16}. The theoretical part of this work studies the optimization landscape of the linear encoder-decoder network in which the encoder is replaced by a truncated butterfly network followed by a dense linear layer in smaller dimension. We call such a network as the encoder-decoder butterfly network. We give an overview of our main result, Theorem \ref{THM:CP}, here. Let $X \in \R^{n\times d}$ and $Y \in \R^{m\times d}$ be the data and output matrices respectively. Then the \emph{encoder-decoder butterfly network} is given as $\overline{Y} = DEBX$, where $D \in \R^{m\times k}$ and $E \in \R^{k \times \ell}$ are dense layers, $B$ is an $\ell\times n$ truncated butterfly network (product of $\log n$ sparse matrices) and $k\leq \ell \leq m\leq n$  (see Section \ref{subsec: encoder-decoder network}). The objective is to learn $D,E$ and $B$ that minimizes $||\overline{Y} - Y||^2_{\text{F}}$. Theorem \ref{THM:CP} shows how the loss at the critical points of such a network depends on the eigenvalues of the matrix $\Sigma = YX^TB^T(BXX^TB^T)^{-1}BXY^T$ \footnote{At a critical point the gradient of the loss function with respect to the parameters in the network is zero.}. In comparison, the loss at the critical points of the encoder-decoder network (without the butterfly network) depends on the eigenvalues of the matrix $\Sigma' = YX^T(XX^T)^{-1}XY^T$ \citep{BaldiH89}. In particular, the loss depends on how the learned matrix $B$ changes the eigenvalues of $\Sigma'$. If we learn only for an optimal $D$ and $E$, keeping $B$ fixed (as done in the experiment in Section \ref{subsec: experiment frozen butterfly}) then it follows from Theorem \ref{THM:CP} that every local minimum is a global minimum and that the loss at the local/global minima depends on how $B$ changes the top $k$ eigenvalues of $\Sigma'$. This inference together with a result by \cite{Sarlos06} is used to give a worst-case guarantee in the special case when $Y=X$ (called auto-encoders that capture PCA; see below Theorem \ref{THM:CP}).

{\bf On the empirical side} -- The outcomes of the following experiments are reported:

(1) In Section \ref{EXP:DENSE}, we replace a dense linear layer in the standard state-of-the-art networks, for both image and language data, with an architecture  that constitutes the composition of (a) truncated butterfly network, (b) dense linear layer in smaller dimension, and (c) transposed truncated butterfly network (see Section \ref{subsec: matrix approx using butterfly network}). The structure parameters are chosen so as to keep the number of weights near linear (instead of quadratic).  

(2) In Sections \ref{EXP:AC} and \ref{subsec: experiment frozen butterfly}, we train a linear encoder-decoder network in which the encoder is replaced by a truncated butterfly network followed by a dense linear layer in smaller dimension. These experiments support our theoretical result. The network structure parameters are chosen so as to keep the number of weights in the (replaced) encoder near linear in the input dimension.  Our results (also theoretically) demonstrate that this has little to no effect on the performance compared to the standard encoder-decoder network.

(3) In Section \ref{subsec: supervise learned butterfly}, we learn the best pre-processing matrix structured as a truncated butterfly network to perform low-rank matrix approximation from a given sample of matrices. We compare our results to that of \cite{IndykVY19}, which learn the pre-processing matrix structured as given in \cite{ClarksonW09}.
\section{Related Work}\label{sec: related work}
Important transforms like discrete Fourier, discrete cosine, Hadamard and many more satisfy a property called \emph{complementary low-rank} property, recently defined by \citet{LiYMHY15}. For an $n\times n$ matrix satisfying this property related to approximation of specific sub-matrices by low-rank matrices, \citet{MichielssenB96} and \citet{NeilWR10} developed the butterfly algorithm to compute the product of such a matrix with a vector in $O(n\log n)$ time. The butterfly algorithm factorizes such a matrix into $O(\log n)$ many matrices, each with $O(n)$ sparsity. In general, the butterfly algorithm has a pre-computation stage which requires $O(n^2)$ time \citep{NeilWR10,Seljebotn2012}.  With the objective of reducing the pre-computation cost \citet{LiYMHY15,LiY17} compute the butterfly factorization for an $n\times n$ matrix satisfying the complementary low-rank property in $O(n^{\frac{3}{2}})$ time. This line of work does not learn butterfly representations for matrices or apply it in neural networks, and is incomparable to our work.

A few papers in the past have used deep learning models with structured matrices (as hidden layers). Such structured matrices can be described using fewer parameters compared to a dense matrix, and hence a representation can be learned by optimizing over a fewer number of parameters. Examples of structured matrices used include low-rank matrices \citep{DenilSDRF13,SainathKSAR13}, circulant matrices \citep{ChengYFKCC15,DingLWLLZWQBYMZ17}, low-distortion projections \citep{YangMDFSSW15}, Toeplitz like matrices \citep{SindhwaniSK15,LuSS16,YeWLCZCX18},  Fourier-related transforms \citep{MoczulskiDAF15} and matrices with low-displacement rank \citep{ThomasGDRR18}. 
It was shown by \citet{LiCL18} that any band-limited function of an input signal can be approximated by applying first a stack of butterfly layers on the signal (giving an approximation of the relevant frequencies of the signal). Our work relies on a different theoretical result (FJLT) that allows approximating any linear mapping 
by a composition of a truncated butterfly, a (small) dense layer and a transposition of a truncated butterfly. Recently, \citet{vahid_2020_CVPR} demonstrated the benefits of replacing the pointwise convolutional layer in CNN's by a butterfly network. Other works by \cite{MocanuESNG18,LeeAT19,WangZG20,VerdeniusSF20} consider a different approach to sparsify neural networks. The work closest to ours are by \cite{YangMDFSSW15}, \cite{MoczulskiDAF15}, and \cite{DaoSGEBLRR20}.

\cite{YangMDFSSW15} and \cite{MoczulskiDAF15} attempt to replace dense linear layers with a stack of structured matrices, including a butterfly structure (the Hadamard or the Cosine transform), but they do not place trainable weights on the edges of the butterfly structure as we do. Note that  adding these trainable weights does not compromise the run time benefits in prediction, while adding to the expressiveness of the network in our case. \cite{DaoSGEBLRR20} replace handcrafted structured sub-networks in machine learning models by a \emph{kaleidoscope} layer, which consists of compositions of butterfly matrices. This is motivated by the fact that the kaleidoscope hierarchy captures a structured matrix exactly and optimally in terms of multiplication operations required to perform the matrix vector product operation. Their work differs from us as we propose to replace any dense linear layer in a neural network (instead of a structured sub-network) by the architecture proposed in Section \ref{subsec: matrix approx using butterfly network}. Our approach is motivated by theoretical results which establish that this can be done with almost no loss in representation. 

Finally, \citet{DaoGERR19} show that butterfly representations of standard transformations like discrete Fourier, discrete cosine, Hadamard mentioned above can be learnt efficiently. They additionally show the following: a) for the benchmark task of compressing a single hidden layer model they compare the network constituting of a composition of butterfly networks with the classification accuracy of a fully-connected linear layer and b) in ResNet a butterfly sub-network is added to get an improved result.  In comparison, our approach to replace a dense linear layer by the proposed architecture in Section \ref{subsec: matrix approx using butterfly network} is motivated by well-known theoretical results as mentioned previously, and the results of the comprehensive list of experiments in Section \ref{EXP:DENSE} support our proposed method.
\section{Proposed Replacement for a Dense Linear Layer}\label{sec: prelim}
In Section \ref{subsec: butterfly network}, we define a truncated butterfly network, and in Section \ref{subsec: matrix approx using butterfly network} we motivate and state our proposed architecture based on truncated butterfly network to replace a dense linear layer in any neural network. All logarithms are in base $2$, and $[n]$ denotes the set $\{1, \ldots , n\}$.
\subsection{Truncated Butterfly Network}\label{subsec: butterfly network}
\begin{definition}[Butterfly Network] \label{definition: butterfly gadget}
Let $n$ be an integral power of $2$. Then an $n\times n$ butterfly network $B$ (see Figure~\ref{figbfnet}) is
a stack of of $\log n$ linear layers, where in each layer $i\in \{0, \ldots ,\log n-1\}$, a bipartite clique
connects between pairs of
nodes $j_1,j_2\in[n]$, for which
the binary representation
of $j_1-1$ and $j_2-1$ differs only in the $i$'th bit.  In particular, the number of edges
in each layer is $2n$.
\end{definition}
In what follows, a \emph{truncated butterfly network} is a butterfly
network in  which
the deepest layer
is truncated, namely, 
only a subset of $\ell$
neurons are kept and
the remaining $n-\ell$
are discarded. The integer
$\ell$ is a tunable parameter, and the choice of
neurons is always assumed
to be sampled uniformly at random and fixed throughout training in what follows.
The effective number of parameters (trainable weights) in a truncated butterfly network is at most $2n\log \ell + 6n$, for any
$\ell$ and any choice of neurons selected from
the last layer.\footnote{Note that if $n$ is not a power of $2$ then we work with the first $n$ columns of the $\ell \times n'$ truncated butterfly network, where $n'$ is the closest number to $n$ that is greater than $n$ and is a power of $2$.}
We include a proof of this simple upper bound in Appendix~\ref{appendix:proof:2nlog2ell} for lack of space
(also, refer to \cite{DBLP:journals/dcg/AilonL09}
for a similar result related to computation time of truncated FFT). 
The reason for studying a truncated butterfly network follows (for example) from the works \citep{AilonC09,DBLP:journals/dcg/AilonL09,KrahmerW2011}.  These papers define randomized linear transformations with the
Johnson-Lindenstrauss property and an efficient computational graph which essentially defines the
truncated butterfly network.  In what follows,
we will collectively denote these constructions
by FJLT.
\footnote{To be precise, the construction in \cite{AilonC09}, \cite{DBLP:journals/dcg/AilonL09}, and \cite{KrahmerW2011}  also uses a random diagonal matrix, but the
values of the diagonal entries can be `absorbed' inside the weights of the first layer of
the butterfly network.}
\subsection{Matrix Approximation Using Butterfly Networks}\label{subsec: matrix approx using butterfly network}
We begin with the following proposition, following known results on matrix approximation (proof in Appendix \ref{secappendix: proof of matrix approx theorem}).
\begin{prop}\label{FACT:MATAPPFJLT}
Suppose $J_1 \in \R^{k_1 \times n_1}$ and $J_2 \in \R^{k_2\times n_2}$ are matrices sampled from FJLT distribution, and let $W \in \R^{n_2\times n_1}$. Then for the random matrix $W' = (J_2^TJ_2)W(J_1^TJ_1)$, any unit vector $\mathbf{x}\in \R^{n_1}$ and any $\epsilon \in (0,1)$,
 $\Pr\left[\|W'\mathbf{x} - W\mathbf{x}\| \leq \epsilon \|W\| \right] \geq 1-e^{-\Omega(\min\{k_1, k_2\}\epsilon^2)}\ $.
\end{prop}

\textbf{Proposed Replacement}: From Proposition \ref{FACT:MATAPPFJLT} it follows that $W'$ approximates the action of $W$ with high
probability on any given input vector.  Now observe that
$W'$ is equal to $J_2^T \tilde W J_1$, where $\tilde{W} = $ $J_2WJ_1^T$. Since $J_1$ and $J_2$ are FJLT, they can be represented by a truncated butterfly network, and hence it is conceivable to replace a dense linear layer connecting $n_1$ neurons to $n_2$ neurons (containing $n_1 n_2$ variables) in any neural network with a composition of three gadgets: a truncated butterfly network of size $k_1\times n_1$, followed by a dense linear layer of size $k_2\times k_1$, followed by the transpose of a truncated butterfly network of size $k_2\times n_2$. In Section \ref{EXP:DENSE}, we replace dense linear layers in common deep learning networks with our proposed architecture, where $k_i << n_i$, $i=1,2$.
\section{Encoder-Decoder Butterfly Network}\label{subsec: encoder-decoder network}
Let $X\in \R^{n\times d}$, and $Y\in \R^{m \times d}$ be  data and output matrices respectively, and $k \leq m\leq n$. Then the encoder-decoder network for $X$ is given as 
$$\overline{Y} = DEX$$
where $E \in \mathbb{R}^{k\times n}$, and $D \in \mathbb{R}^{m \times k}$ are called the encoder and decoder matrices respectively. For the special case when $Y=X$, it is called auto-encoders. The optimization problem is to learn matrices $D$ and $E$ such that $||Y-\overline{Y}||^2_{\text{F}}$ is minimized. The optimal solution is denoted as $Y^*, D^*$ and $E^*$\footnote{Possibly multiple $D^{*}$ and $E^*$ exist such that $Y^{*} = D^{*}E^{*}X$.}. In the case of auto-encoders $X^* = X_k$, where $X_k$ is the best rank $k$ approximation of $X$. In this section, we study the optimization landscape of the encoder-decoder butterfly network : an encoder-decoder network, where the encoder is replaced by a truncated butterfly network followed by a dense linear layer in smaller dimension. 
Such a replacement is motivated by the following result from \cite{Sarlos06}, in which $\Delta_k = ||X_k -X||^{2}_{\text{F}}$. 
\begin{prop}\label{fact: rank k approx after pre processing}
Let $X\in \R^{n\times d}$. Then with probability at least $1/2$, the best rank $k$ approximation of $X$ from the rows of $JX$ (denoted $J_k(X)$), where $J$ is sampled from an $\ell\times n$ FJLT distribution and $\ell = (k\log k +k/\epsilon)$ satisfies $||J_k(X) - X||^{2}_{\text{F}} \leq (1+\epsilon) \Delta_k$. 
\end{prop}
Proposition \ref{fact: rank k approx after pre processing} suggests that in the case of auto-encoders we could replace the encoder with a truncated butterfly network of size $\ell \times n$ followed by a dense linear layer of size $k\times \ell$, and obtain a network with fewer parameters but loose very little in terms of representation. Hence, it is worthwhile investigating the representational power of the encoder-decoder butterfly network
\begin{equation}\label{equation: neural network Y}
\overline{Y} = DEBX    \ .
\end{equation} 
\begin{table*}
    \centering
    \scriptsize
    \begin{tabular}{|c|c|c|c|c|}
    \hline
        Dataset Name & Task & Model   \\
        \hline
         Cifar-10 \cite{Krizhevsky09learningmultiple} & Image classification & EfficientNet \cite{TanL19}     \\
         Cifar-10 \cite{Krizhevsky09learningmultiple} & Image classification & PreActResNet18 \cite{HeZRS16}    \\
         Cifar-100 \cite{Krizhevsky09learningmultiple} & Image classification & seresnet152 \cite{HuSASW20}   \\
         Imagenet \cite{imagenet_cvpr09}& Image classification &  senet154 \cite{HuSASW20} \\
         CoNLL-03  \cite{CONLL} & Named Entity Recognition (English) & Flair's Sequence Tagger \cite{flair} \cite{flairembeddings}  \\
         CoNLL-03 \cite{CONLL}   & Named Entity Recognition (German) & Flair's Sequence Tagger \cite{flair} \cite{flairembeddings}  \\
         Penn Treebank (English) \cite{treebank}  & Part-of-Speech Tagging  & Flair's Sequence Tagger \cite{flair} \cite{flairembeddings}  \\
         \hline
    \end{tabular}
    \caption{Data and the corresponding architectures used in the fast matrix multiplication using butterfly matrices experiments.}
    \label{tab:exp4_data}
\end{table*}
Here, $X$, $Y$ and $D$ are as in the encoder-decoder network, $E \in \R^{k\times \ell}$ is a dense matrix, and $B$ is an $\ell\times n$ truncated butterfly network. In the encoder-decoder butterfly network the encoding is done using $EB$, and decoding is done using $D$. This reduces the number of parameters in the encoding matrix from $kn$ (as in the encoder-decoder network) to $k\ell + O(n\log \ell)$. Again the objective is to learn matrices $D$ and $E$, and the truncated butterfly network $B$ such that $||Y-\overline{Y}||^2_{\text{F}}$ is minimized. The optimal solution is denoted as $Y^*, D^*$, $E^*$, and $B^*$. Theorem \ref{THM:CP} shows that the loss at a critical point of such a network depends on the eigenvalues of $\Sigma(B) = YX^TB^T(BXX^TB^T)^{-1}XY^T$, when $BXX^TB^T$ is invertible and $\Sigma(B)$ has $\ell$ distinct positive eigenvalues.The loss $\mathcal{L}$ is defined as $||\overline{Y} - Y||^{2}_{\text{F}}$.
\begin{theorem}\label{THM:CP}
Let $D, E$ and $B$ be a point of the encoder-decoder network with a truncated butterfly network satisfying the following: a) $BXX^TB^T$ is invertible, b) $\Sigma(B)$ has $\ell$ distinct positive eigenvalues $\lambda_1 > \ldots > \lambda_{\ell}$, and c) the gradient of $\mathcal{L}(\overline{Y})$ with respect to the parameters in $D$ and $E$ matrix is zero. Then corresponding to this point (and hence corresponding to every critical point) there is an $I \subseteq [\ell]$ such that $\mathcal{L}(\overline{Y})$ at this point is equal to $\text{tr}(YY^T) - \sum_{i\in I}\lambda_i$. Moreover if the point is a local minima then $I = [k]$.
\end{theorem}
The proof of Theorem \ref{THM:CP} is given in Appendix \ref{secappendix: proof of theorem}. As discussed in \citep{Kawaguchi16}, the assumptions of having full rank and distinct eigenvalues in the training data matrix $X$ (see Theorem 2.3 in \cite{Kawaguchi16}) are realistic and practically easy to satisfy. We require the same assumptions on $BX$ (instead), where $B$ is sampled from an FJLT distribution. We also compare our result with that of \cite{BaldiH89} and \cite{Kawaguchi16}, which study the optimization landscape of dense linear neural networks in Appendix \ref{secappendix: proof of theorem}. From Theorem \ref{THM:CP} it follows that if $B$ is fixed and only $D$ and $E$ are trained then a local minima is indeed a global minima. We use this to claim a worst-case guarantee using a two-phase learning approach to train an auto-encoder. In this case the optimal solution is denoted as $B_k(Y), D_{B},$ and $E_B$. Observe that when $Y=X$, $B_k(X)$ is the best rank $k$ approximation of $X$ computed from the rows of $BX$. 

\textbf{Two phase learning for auto-encoder}: Let $\ell = k\log k + k/\epsilon$ and consider a two phase learning strategy for auto-encoders, as follows: In phase one $B$ is sampled from an FJLT distribution, and then only $D$ and $E$ are trained keeping $B$ fixed. Suppose the algorithm learns $D'$ and $E'$ at the end of phase one, and $X' = D'E'B$. Then Theorem \ref{THM:CP} guarantees that, assuming $\Sigma(B)$ has $\ell$ distinct positive eigenvalues and $D', E'$ are a local minima, $D' = D_B$, $E'= E_B$, and $X'= B_k(X)$. Namely $X'$ is the best rank $k$ approximation of $X$ from the rows of $BX$. From Proposition \ref{fact: rank k approx after pre processing} with probability at least $\frac{1}{2}$, $\mathcal{L}(X') \leq (1+\epsilon)\Delta_k$. In the second phase all three matrices are trained to improve the loss.  In Sections \ref{EXP:AC} and \ref{subsec: experiment frozen butterfly} we train an encoder-decoder butterfly network using the standard gradient descent method. In these experiments the truncated butterfly network is initialized by sampling it from an FJLT distribution, and $D$ and $E$ are initialized randomly as in Pytorch. 
\section{Experiments}\label{sec: experiments on replacing layers in neural nets}
In this section we report the experimental results based on the ideas presented in Sections \ref{subsec: matrix approx using butterfly network} and \ref{subsec: encoder-decoder network}. 
%
The code for our experiments is publicly available (see \cite{AilonLN_code}).
\subsection{Replacing Dense Linear Layers by the Proposed Architecture}\label{EXP:DENSE}
This experiment replaces a dense linear layer of size $n_2\times n_1$ in common deep learning architectures with the network proposed in Section \ref{subsec: matrix approx using butterfly network}.\footnote{In all the architectures considered the final linear layer before the output layer is replaced, and $n_1$ and $n_2$ depend on the architecture.} The truncated butterfly networks are initialized by sampling it from the FJLT distribution, and the dense matrices are initialized randomly as in Pytorch. We set $k_1 =\log n_1$ and $k_2 = \log n_2$. The datasets and the corresponding architectures considered are summarized in Table \ref{tab:exp4_data}. 
\begin{figure*}[t]
	\centering
	\includegraphics[width=0.45\textwidth]{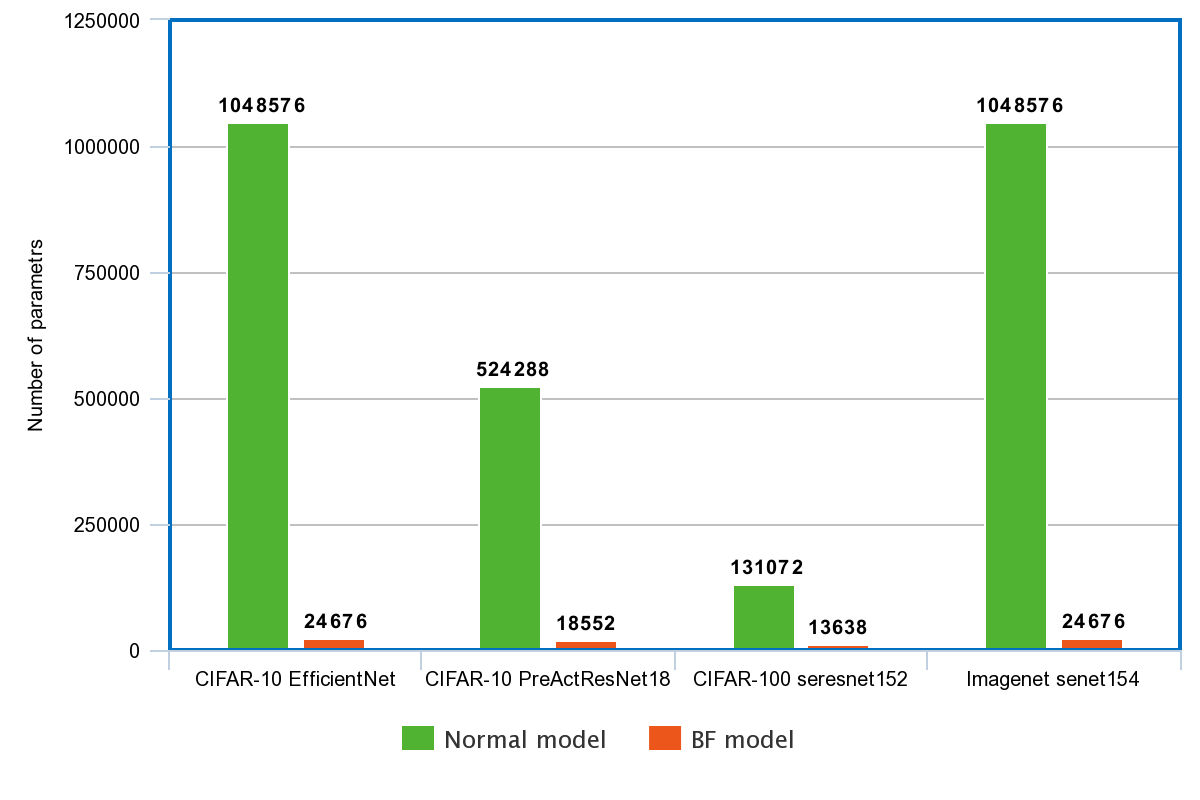} \qquad
	\includegraphics[width=0.45\textwidth]{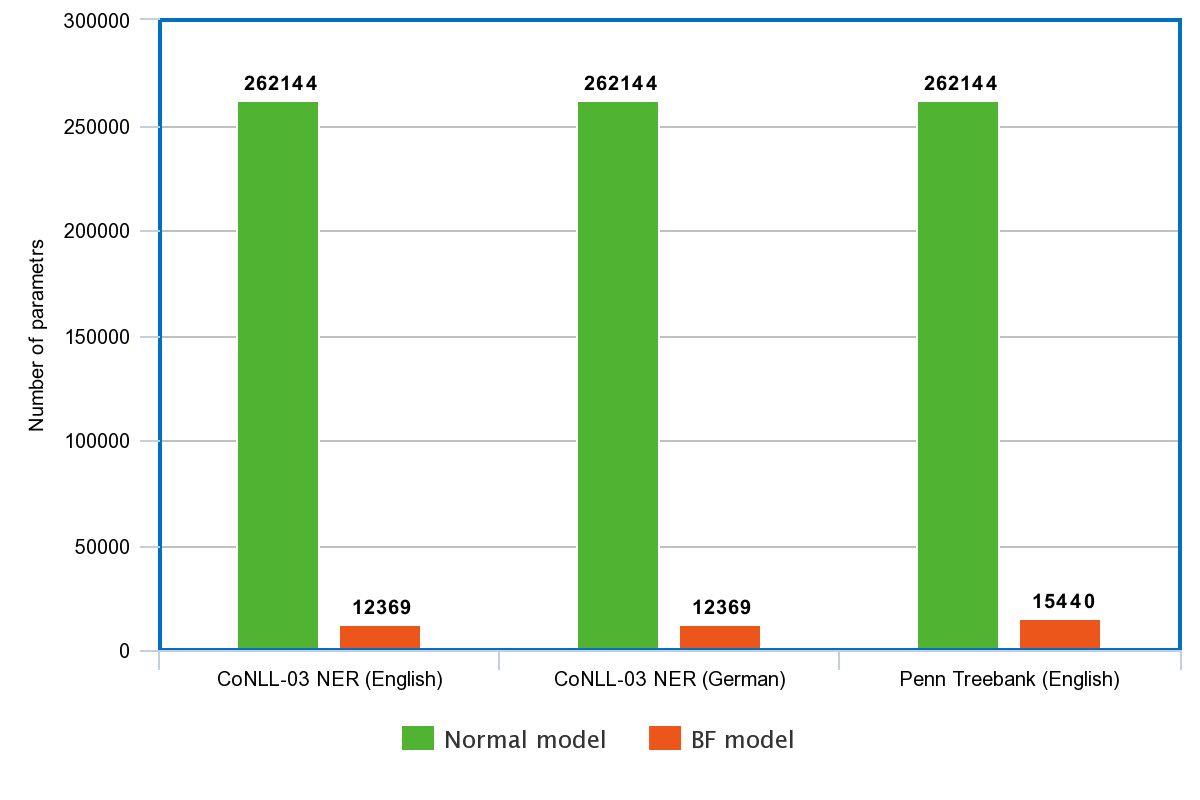} 
	\caption{Number of parameters in the dense linear layer of the original model and in the replaced butterfly based architecture; Left: Vision data, Right: NLP
	}
	\label{figurenumberofparams nlp}
\end{figure*}

For each dataset and model, the objective function is the same as defined in the model, and the generalization and convergence speed between the original model and the modified one (called the butterfly model for convenience) are compared. Figure \ref{figurenumberofparams nlp} reports the number of parameters in the dense linear layer of the original model, and in the replaced network, and Figure \ref{figurenumberofparams complete model} in Appendix \ref{subsecappendix: plots from dense} displays the number of parameter in the original model and the butterfly model. In particular, Figure \ref{figurenumberofparams nlp} shows the significant reduction in the number of parameters obtained by the proposed replacement.
\begin{figure}[!htb]
\begin{center}
  \includegraphics[width =0.7\linewidth]{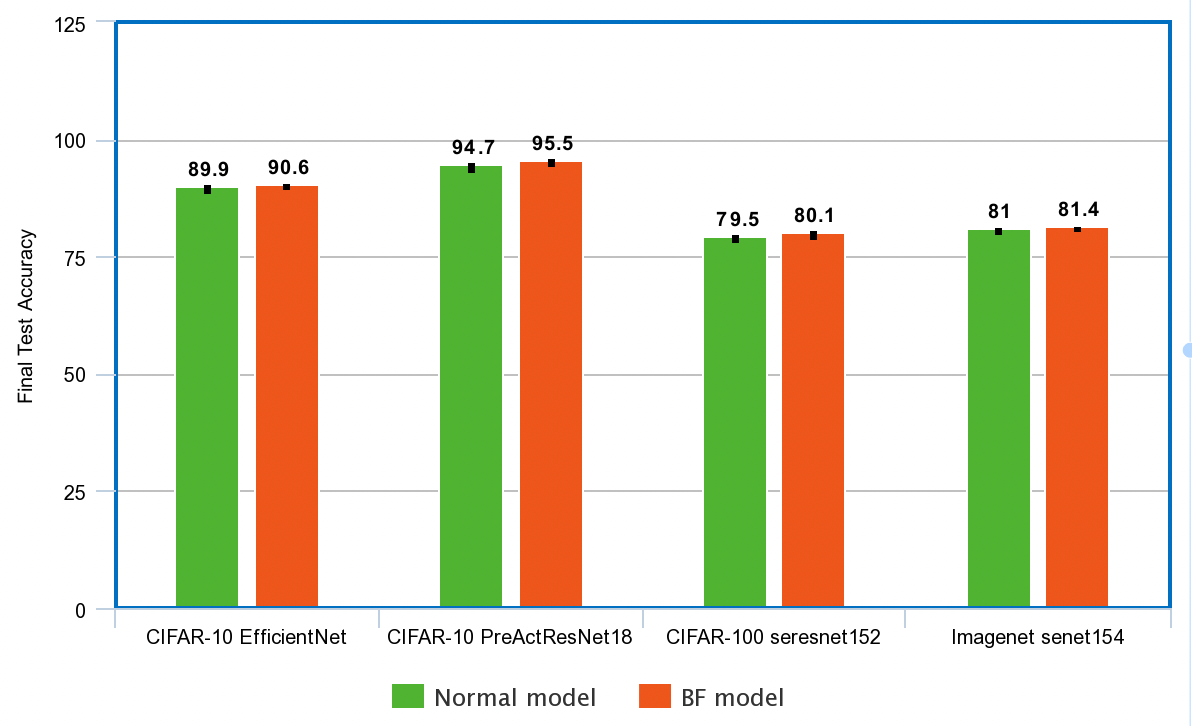}
  \caption{Comparison of final test accuracy with different image classification models and data sets}
  \label{figure: comparing final test accuracy}
 \end{center}
\end{figure}

\begin{figure}[!htb]
\begin{center}
  \includegraphics[width=0.7\linewidth]{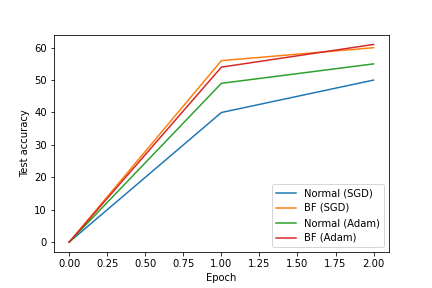}
\caption{Comparison of test accuracy in the first few epochs with different models and optimizers on CIFAR-10 with PreActResNet18}
  \label{figureimgclass}

\end{center}\end{figure}

Figure \ref{figure: comparing final test accuracy} reports the test accuracy of the original model and the butterfly model. The black vertical lines in Figures \ref{figure: comparing final test accuracy} denote the error bars corresponding to standard deviation, and the values above the rectangles denote the average accuracy.  In Figure \ref{figureimgclass} observe that the test accuracy for the butterfly model trained with stochastic gradient descent is even better than the original model trained with Adam in the first few epochs.
Figure \ref{figure20epochs} in Appendix \ref{subsecappendix: plots from dense} compares the test accuracy in the the first 20 epochs of the original and butterfly model. The results for the NLP tasks in the interest of space are reported in Figure \ref{figureNLP}, Appendix \ref{subsecappendix: plots from dense}. The training and inference times required for the original model and the butterfly model in each of these experiments are reported in Figures \ref{figuretrainingtime} and \ref{figuretrainingtime for nlp} in Appendix \ref{subsecappendix: plots from dense}. We remark that the modified architecture is also trained for fewer epochs. In almost all the cases the modified architecture does better than the normal architecture, both in the rate of convergence and in the final accuracy/$F1$ score. Moreover, the training time for the modified architecture is less.

\subsection{Encoder-Decoder Butterfly network with Synthetic Gaussian and Real Data}\label{EXP:AC}
This experiment tests whether gradient descent based techniques can be used to train an auto-encoder with a truncated butterfly gadget (see Section \ref{subsec: encoder-decoder network}). Five types of data matrices are tested: two are random and three are constructed using standard public real image datasets. For the matrices constructed from the image datasets, the input coordinates are randomly permuted, which ensures the network cannot take advantage of the spatial structure in the data.

Table \ref{tab:exp1_data} summarizes the data attributes. Gaussian $1$ and Gaussian $2$ are Gaussian matrices with rank $32$ and $64$ respectively. A Rank $r$ Gaussian matrix is constructed as follows: $r$ orthogonal vectors of size $1024$ are sampled at random and the columns of the matrix are determined by taking random linear combinations of these vectors, where the coefficients are chosen independently and uniformly at random from the Gaussian distribution with mean $0$ and variance $0.01$. The data matrix for MNIST is constructed as follows: each row corresponds  to an image represented as a  $28\times 28$ matrix (pixels) sampled uniformly at random from the MNIST database of handwritten digits \cite{MNIST2010} which is extended to a $32 \times 32$ matrix by padding numbers close to zero and then represented as a vector of size $1024$ in column-first ordering\footnote{Close to zero entries are sampled uniformly at random according to a Gaussian distribution with mean zero and variance $0.01$.}. Similar to the MNIST every row of the data matrix for Olivetti corresponds to an image represented as a $64 \times 64$ matrix sampled uniformly at random from the Olivetti faces data set \cite{Olivetti1992}, which is represented as a vector of size $4096$ in column-first ordering. Finally, for HS-SOD the data matrix is a $1024 \times 768$ matrix sampled uniformly at random from HS-SOD -- a dataset for hyperspectral images from natural scenes \cite{Hyper2018}. 

\begin{table}[ht!]
    \centering
    \begin{tabular}{|c|c|c|c|c|}
    \hline
        Name & $n$ & $d$ & rank  \\
        \hline
         Gaussian 1 & 1024 & 1024 & 32   \\
         Gaussian 2 & 1024 & 1024 & 64 \\
         MNIST & 1024 & 1024 & 1024 \\
         Olivetti & 1024 & 4096 & 1024  \\
         HS-SOD & 1024 & 768 & 768 \\
         \hline
    \end{tabular}
    \caption{Data used in the truncated butterfly auto-encoder reconstruction experiments}
    \label{tab:exp1_data}
\end{table}
For each of the data matrices the loss obtained via training the truncated butterfly network with the Adam optimizer is compared to $\Delta_k$ (denoted as PCA) and $||J_k(X) - X||^{2}_{\text{F}}$ where $J$ is an $\ell \times n$ matrix sampled from the FJLT distribution (denoted as FJLT+PCA).\footnote{PCA stands for principal component analysis which is a standard way to compute $X_k$.} Figures \ref{figapproxerrorgauss} and \ref{figapproxerrormnist} reports the loss on Gaussian 1 and MNIST respectively, whereas Figure \ref{figapproxappendix} in Appendix \ref{subsecappendix: plots from truncated butterfly} reports the loss for the remaining data matrices. Observe that for all values of $k$ the loss for the encoder-decoder butterfly network is almost equal to $\Delta_k$, and is in fact $\Delta_k$ for small and large values of $k$.
\begin{figure}[!htb]
\begin{center}
  \includegraphics[width=0.7\linewidth]{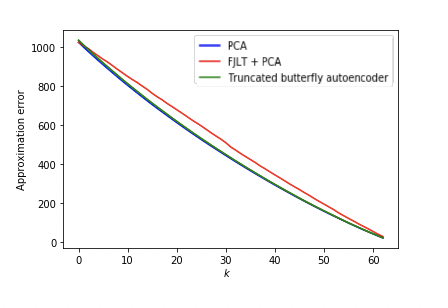}
  \caption{Approximation error on data matrix with various methods for various values of $k$ (Gaussian 1)}
  \label{figapproxerrorgauss}
\end{center}
\end{figure}
\begin{figure}[!htb]
\begin{center}
  \includegraphics[width=0.7\linewidth]{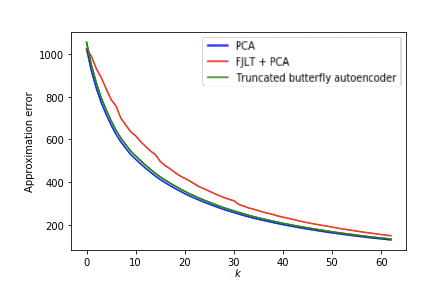}
\caption{Approximation error on data matrix with various methods for various values of $k$ (MNIST)}
\label{figapproxerrormnist}
\end{center}
\end{figure}

%
\subsection{Two-phase Learning}\label{subsec: experiment frozen butterfly}
This experiment is similar to the experiment in Section \ref{EXP:AC} but the training in this case is done in two phases. In the first phase, $B$ is fixed and the network is trained to determine an optimal $D$ and $E$. In the second phase, the optimal $D$ and $E$ determined in phase one are used as the initialization, and the network is trained over $D, E$ and $B$ to minimize the loss. Theorem \ref{THM:CP} ensures worst-case guarantees for this two phase training (see below the theorem). Figure \ref{figure: two phase training} reports the approximation error of an image from Imagenet. The red and green lines in Figure \ref{figure: two phase training} correspond to the approximation error at the end of phase one and two respectively. 
\begin{figure}[!htb]
\begin{center}
  \includegraphics[width=0.7\linewidth]{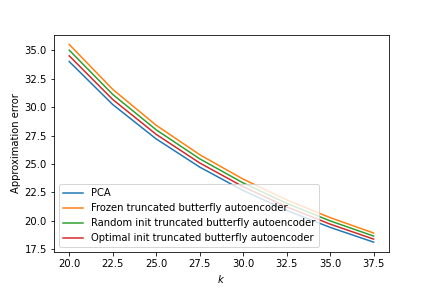}
  \caption{Approximation error on data matrix with various methods for various values of $k$ (Gaussian 1)}
\label{figure: two phase training}
\end{center}
\end{figure}

%
\section{Sketching for Low-Rank Matrix Decomposition}\label{subsec: supervise learned butterfly}
This experiment was inspired by the recent influential work by \cite{IndykVY19}, which considers a supervised learning approach to compute an  $\ell\times n$ pre-conditioning matrix $B$, where $\ell \ll n$, such that for $X \in \R^{n\times d}$, the best rank $k$ approximation of $X$ from the rows of $BX$ (denoted $B_{k}(X)$) is optimized. The matrix $B$  has a fixed sparse structure determined a priory as in \citep{ClarksonW09}, and the non-zero entries are learned to minimize the loss over a training set of matrices. The results in \citep{IndykVY19} suggest that a learned matrix $B$ significantly improves the guarantee compared to a random sketching matrix as in \citep{ClarksonW09}. Our setting is similar to that in \citep{IndykVY19}, except that $B$ is now represented as an $\ell \times n$ truncated butterfly gadget. Our experiments on several datasets show that indeed a learned truncated butterfly gadget does better than a random matrix, and even a learned $B$ as in \citep{IndykVY19}.


\textbf{Setup}: Suppose $X_1, \ldots , X_t \in \R^{n\times d}$ are training matrices sampled from a distribution $\mathcal{D}$. Then a $B$ is computed that minimizes the following empirical loss
\begin{equation}
    \sum_{i\in [t]} ||X_i - B_k(X_i)||^{2}_{\text{F}}
\end{equation}
We compute $B_k(X_i)$ using truncated SVD of $BX_i$ (as in Algorithm 1, \citep{IndykVY19}). The matrix $B$ is learned by the back-propagation algorithm that uses a differentiable SVD implementation to calculate the gradients, followed by optimization with Adam such that the butterfly structure of $B$ is maintained. The learned $B$ can be used as the pre-processing matrix for any matrix in the future. The test error for a matrix $B$ and a test set $\mathsf{Te}$ is defined as follows: 
$$ \text{Err}_{\mathsf{Te}}(B) = \mathbf{E}_{X \sim \mathsf{Te}}\left[\, ||X-B_k(X)||^{2}_{\text{F}} \, \right] - \text{App}_{\mathsf{Te}}, $$
where $\text{App}_{\mathsf{Te}} = \mathbf{E}_{X \sim \mathsf{Te}}[\ ||X-X_k||^{2}_{\text{F}}\ ]$.

\textbf{Experiments and Results}: The experiments are performed on the datasets shown in Table \ref{tab:exp3_data}. In HS-SOD \citep{Hyper2018} and CIFAR-10 \citep{Krizhevsky09learningmultiple} 400 training matrices ($t=400$), and 100 test matrices are sampled, while in Tech 200 \citep{Tech2004}, training matrices ($t=200$), and 95 test matrices are sampled. In Tech, each matrix has 835,422 rows but on average only 25,389 rows and 195 columns contain non-zero entries. For the same reason as in Section \ref{EXP:AC}  in each dataset, the coordinates of each row are randomly permuted. Some of the matrices in the datasets have much larger singular values than the others, and to avoid
imbalance in the dataset, the matrices are normalized so that their top singular values are all equal, as done in \citep{IndykVY19}. 
\begin{table}[ht!]
    \centering
    \begin{tabular}{|c|c|c|c|c|}
    \hline
        Name & $n$ & $d$  \\
        \hline
         HS-SOD 1 & 1024 & 768   \\
         CIFAR-10  & 32 & 32 \\
         Tech & 25,389 & 195   \\
         \hline
    \end{tabular}
    \caption{Data used in the Sketching algorithm for low-rank matrix decomposition experiments.}
    \label{tab:exp3_data}
\end{table}
For each of the datasets, the test error for the learned $B$ via our truncated butterfly structure is compared to the test errors for the following three cases: 1) $B$ is a learned as a sparse sketching matrix as in \cite{IndykVY19}, b) $B$ is a random sketching matrix as in \cite{ClarksonW09}, and c) $B$ is an $\ell\times n$  Gaussian matrix. Figure \ref{figurek10m20} compares the test error for $\ell=20$, and $k=10$, where $\text{App}_{\mathsf{Te}} = 10.56$. Figure \ref{figtesterrl20k1} in Appendix \ref{subsecappendix: plots from supervise learned butterfly} compares the test errors of the different methods in the extreme case when $k=1$, and Figure \ref{figtesterrork10} in Appendix \ref{subsecappendix: plots from supervise learned butterfly} compares the test errors of the different methods for various values of $\ell$. Table \ref{figuretable} in Appendix \ref{subsecappendix: plots from supervise learned butterfly} reports the test error for different values of $\ell$ and $k$. Figure \ref{figurelossperstep} in Appendix \ref{subsecappendix: plots from supervise learned butterfly} shows the test error for $\ell=20$ and $k=10$ during the training phase on HS-SOD. In Figure \ref{figurelossperstep} it is observed that the butterfly learned is able to surpass sparse learned after merely a few iterations.
\begin{figure}[!htb]
\minipage{0.45\textwidth}
  \includegraphics[width=\linewidth]{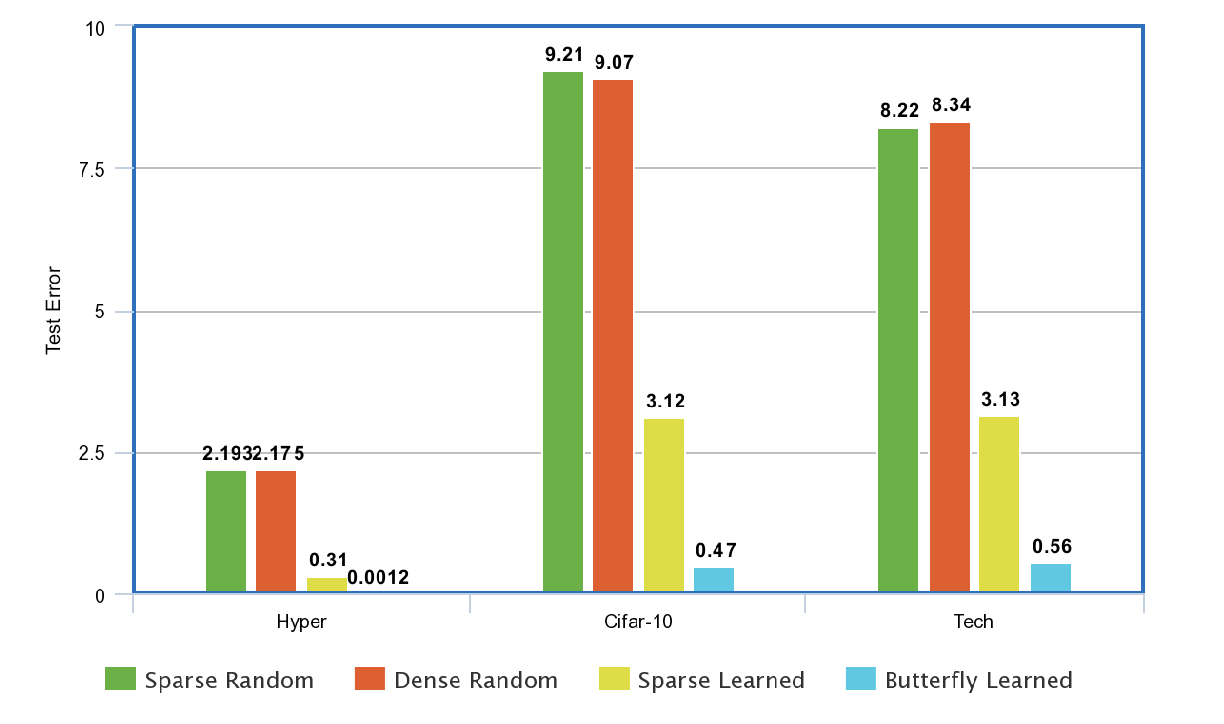}
  \caption{Test error by different sketching matrices on different data sets}
  \label{figurek10m20}
\endminipage\hfill
\minipage{0.45\textwidth}
  \includegraphics[width=\linewidth]{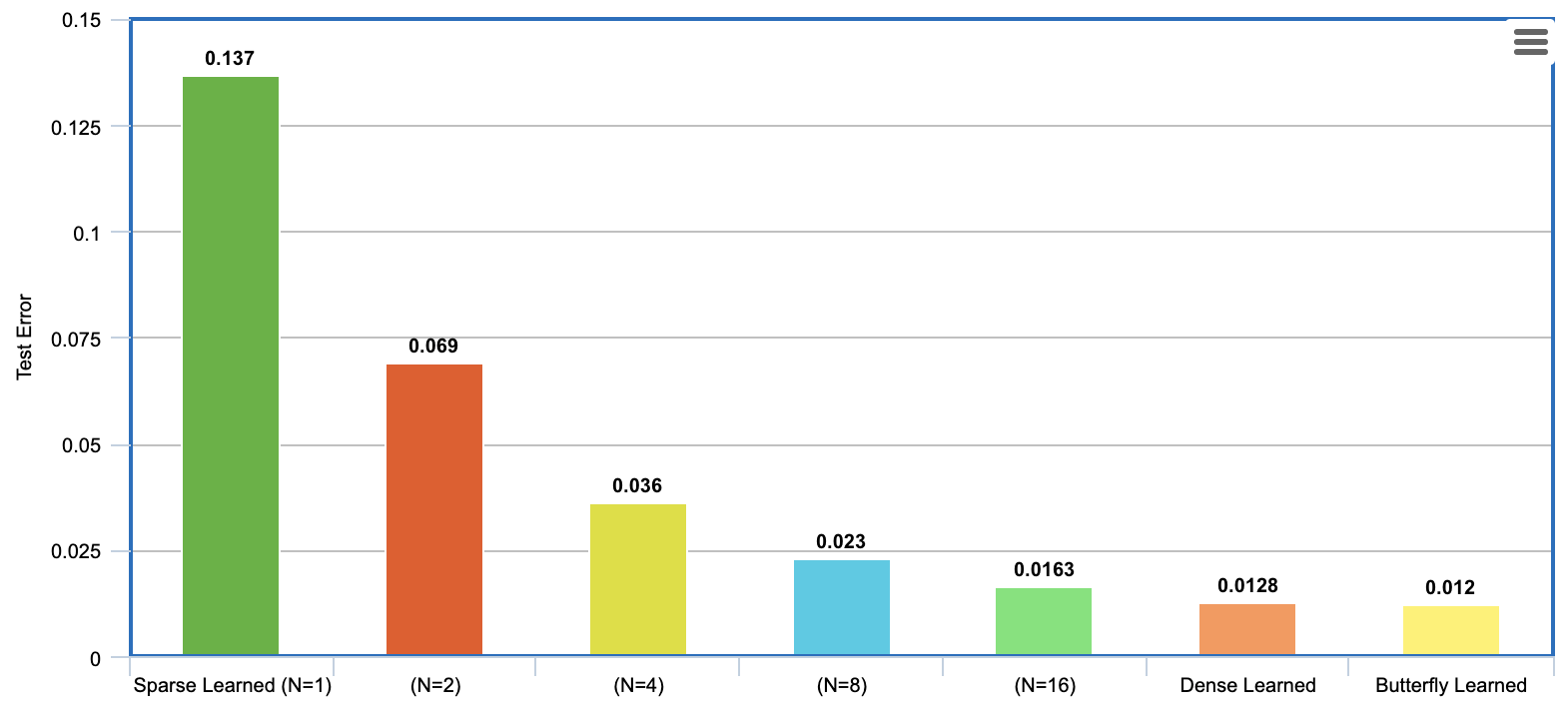}
  \caption{Test errors for various values of $N$ and a learned butterfly matrix}
  \label{figuretradeoff}
\endminipage\hfill
\end{figure}

Figure \ref{figuretradeoff} compares  the test error for the learned $B$ via our truncated butterfly structure to a learned matrix $B$  with $N$ non-zero entries in each column -- the $N$ non-zero location for each column are chosen uniformly at random. The reported test errors are on HS-SOD, when $\ell =20$ and $k = 10$. 
Interestingly, the error for butterfly learned is not only less than the error for sparse learned ($N=1$ as in \citep{IndykVY19}) but also less than than the error for dense learned ($N=20$). In particular, our results indicate that using a learned butterfly sketch can significantly reduce the approximation loss compared to using a learned sparse sketching matrix.

\section{Conclusion}\label{sec: conclusion}

\textbf{Discussion}: Among other things,
this work showed that it is beneficial to replace  dense  linear layer in  deep learning architectures with a more compact 
architecture (in terms of number of parameters), using truncated butterfly networks.  This 
approach is justified using
ideas from efficient matrix approximation theory from the last two decades. 
however, results in additional logarithmic depth to the network.  This issue raises the question of whether
the extra depth may harm convergence of gradient descent optimization.  
To start answering this question, we show, both empirically
and theoretically, that
in linear encoder-decoder networks
in which the encoding is done using a butterfly network, this typically does not happen.
To further demonstrate the utility of truncated butterfly networks, we consider a supervised learning approach as in \cite{IndykVY19}, where we learn
how to 
derive low rank approximations
of a distribution of matrices by multiplying a pre-processing linear operator represented as a butterfly network,
with weights trained using a sample of the distribution.

\textbf{Future Work}: The main open questions arising from the work are related to better
understanding the optimization landscape of butterfly networks.
The current tools for analysis of
deep linear networks do not apply
for these structures, and more
theory is necessary. It would be interesting to determine whether replacing dense linear layers in any network, with butterfly networks as in Section~\ref{subsec: matrix approx using butterfly network} \emph{harms} the convergence
of the original matrix. Another direction would be to check empirically whether adding non-linear gates between the layers (logarithmically many) of a butterfly network improves the performance of the network. In the experiments in Section \ref{EXP:DENSE}, we have replaced a single dense layer by our proposed architecture. It would be worthwhile to check whether replacing multiple dense linear layers in the different architectures harms the final accuracy. 
Similarly, it might be insightful to replace a convolutional layer by an architecture based on truncated butterfly network. Finally, since our proposed replacement reduces the number of parameters in the network, it might be possible to empirically show that the new network is more resilient to over-fitting.
%
\section*{Acknowledgements} 
 This project has received funding from European Union’s Horizon 2020 research and innovation program under grant agreement No 682203 -ERC-[ Inf-Speed-Tradeoff].

\bibliographystyle{plainnat}

\bibliography{ailon_467}

\begin{thebibliography}{48}
\providecommand{\natexlab}[1]{#1}
\providecommand{\url}[1]{\texttt{#1}}
\expandafter\ifx\csname urlstyle\endcsname\relax
  \providecommand{\doi}[1]{doi: #1}\else
  \providecommand{\doi}{doi: \begingroup \urlstyle{rm}\Url}\fi

\bibitem[Ailon and Chazelle(2009)]{AilonC09}
Nir Ailon and Bernard Chazelle.
\newblock The fast johnson--lindenstrauss transform and approximate nearest
  neighbors.
\newblock \emph{{SIAM} Journal on Computing}, 39\penalty0 (1):\penalty0
  302--322, 2009.

\bibitem[Ailon and Liberty(2009)]{DBLP:journals/dcg/AilonL09}
Nir Ailon and Edo Liberty.
\newblock Fast dimension reduction using rademacher series on dual {BCH} codes.
\newblock \emph{Discrete and Computational Geometry}, 42\penalty0 (4):\penalty0
  615--630, 2009.

\bibitem[Ailon et~al.(2021)Ailon, Leibovitch, and Nair]{AilonLN_code}
Nir Ailon, Omer Leibovitch, and Vineet Nair.
\newblock {Code for Sparse Linear Networks with a Fixed Butterfly Structure:
  Theory and Practice}.
\newblock \url{https://github.com/leibovit/Sparse-Linear-Networks}, 2021.

\bibitem[Akbik et~al.(2018)Akbik, Blythe, and Vollgraf]{flair}
Alan Akbik, Duncan Blythe, and Roland Vollgraf.
\newblock Contextual string embeddings for sequence labeling.
\newblock In \emph{International Conference on Computational Linguistics
  {COLING}}, pages 1638--1649, 2018.

\bibitem[Akbik et~al.(2019)Akbik, Bergmann, and Vollgraf]{flairembeddings}
Alan Akbik, Tanja Bergmann, and Roland Vollgraf.
\newblock Pooled contextualized embeddings for named entity recognition.
\newblock In \emph{Conference of the North American Chapter of the Association
  for Computational Linguistics {NAACL}}, page 724–728, 2019.

\bibitem[Alizadeh et~al.(2020)Alizadeh, Anish, Ali, and
  Mohammad]{vahid_2020_CVPR}
Keivan Alizadeh, Prabhu Anish, Farhadi Ali, and Rastegari Mohammad.
\newblock Butterfly transform: An efficient fft based neural architecture
  design.
\newblock In \emph{Conference on Computer Vision and Pattern Recognition
  (CVPR)}, 2020.

\bibitem[Baldi and Hornik(1989)]{BaldiH89}
Pierre Baldi and Kurt Hornik.
\newblock Neural networks and principal component analysis: Learning from
  examples without local minima.
\newblock \emph{Neural Networks}, 2\penalty0 (1):\penalty0 53--58, 1989.

\bibitem[Cambridge(1994)]{Olivetti1992}
AT\&T~Laboratories Cambridge.
\newblock The olivetti faces dataset, 1994.

\bibitem[Cheng et~al.(2015)Cheng, Yu, Feris, Kumar, Choudhary, and
  Chang]{ChengYFKCC15}
Yu~Cheng, Felix~X. Yu, Rog{\'{e}}rio~Schmidt Feris, Sanjiv Kumar, Alok~N.
  Choudhary, and Shih{-}Fu Chang.
\newblock An exploration of parameter redundancy in deep networks with
  circulant projections.
\newblock In \emph{International Conference on Computer Vision, {ICCV}}, pages
  2857--2865. {IEEE} Computer Society, 2015.

\bibitem[Clarkson and Woodruff(2009)]{ClarksonW09}
Kenneth~L. Clarkson and David~P. Woodruff.
\newblock Numerical linear algebra in the streaming model.
\newblock In Michael Mitzenmacher, editor, \emph{Proceedings of the 41st Annual
  {ACM} Symposium on Theory of Computing, {STOC} 2009}, pages 205--214. {ACM},
  2009.

\bibitem[Cooley and Tukey(1965)]{CooleyT65}
J.W. Cooley and J.W. Tukey.
\newblock An algorithm for the machine calculation of complex fourier series.
\newblock \emph{Mathematics of Computation}, 19\penalty0 (90):\penalty0
  297--301, 1965.

\bibitem[Dao et~al.(2019)Dao, Gu, Eichhorn, Rudra, and R{\'{e}}]{DaoGERR19}
Tri Dao, Albert Gu, Matthew Eichhorn, Atri Rudra, and Christopher R{\'{e}}.
\newblock Learning fast algorithms for linear transforms using butterfly
  factorizations.
\newblock In Kamalika Chaudhuri and Ruslan Salakhutdinov, editors,
  \emph{International Conference on Machine Learning, {ICML}}, volume~97 of
  \emph{Proceedings of Machine Learning Research}, pages 1517--1527. {PMLR},
  2019.

\bibitem[Dao et~al.(2020)Dao, Sohoni, Gu, Eichhorn, Blonder, Leszczynski,
  Rudra, and {\ '{e}}]{DaoSGEBLRR20}
Tri Dao, Nimit~Sharad Sohoni, Albert Gu, Matthew Eichhorn, Amit Blonder, Megan
  Leszczynski, Atri Rudra, and Christopher~R {\ '{e}}.
\newblock Kaleidoscope: An efficient, learnable representation for all
  structured linear maps.
\newblock In \emph{International Conference on Learning Representations,
  {ICLR}}, 2020.

\bibitem[Davido et~al.(2004)Davido, Gabrilovich, and Markovitch]{Tech2004}
D.~Davido, E.~Gabrilovich, and S.~Markovitch.
\newblock Parameterized generation of labeled datasets for text categorization
  based on a hierarchical directory.
\newblock In \emph{International {ACM SIGIR} Conference on Research and
  Development in Information Retrieval, {SIGIR}}, pages 250--257, 2004.

\bibitem[Deng et~al.(2009)Deng, Dong, Socher, Li, Li, and
  Fei-Fei]{imagenet_cvpr09}
J.~Deng, W.~Dong, R.~Socher, L.-J. Li, K.~Li, and L.~Fei-Fei.
\newblock {ImageNet: A Large-Scale Hierarchical Image Database}.
\newblock In \emph{Conference on Computer Vision and Pattern Recognition,
  {CVPR}}. {IEEE} Computer Society, 2009.

\bibitem[Denil et~al.(2013)Denil, Shakibi, Dinh, Ranzato, and
  de~Freitas]{DenilSDRF13}
Misha Denil, Babak Shakibi, Laurent Dinh, Marc'Aurelio Ranzato, and Nando
  de~Freitas.
\newblock Predicting parameters in deep learning.
\newblock In \emph{Advances in Neural Information Processing Systems
  {NeurIPS}}, pages 2148--2156, 2013.

\bibitem[Ding et~al.(2017)Ding, Liao, Wang, Li, Liu, Zhuo, Wang, Qian, Bai,
  Yuan, Ma, Zhang, Tang, Qiu, Lin, and Yuan]{DingLWLLZWQBYMZ17}
Caiwen Ding, Siyu Liao, Yanzhi Wang, Zhe Li, Ning Liu, Youwei Zhuo, Chao Wang,
  Xuehai Qian, Yu~Bai, Geng Yuan, Xiaolong Ma, Yipeng Zhang, Jian Tang, Qinru
  Qiu, Xue Lin, and Bo~Yuan.
\newblock Circnn: accelerating and compressing deep neural networks using
  block-circulant weight matrices.
\newblock In \emph{{IEEE/ACM} International Symposium on Microarchitecture,
  {MICRO}}, pages 395--408. {ACM}, 2017.

\bibitem[He et~al.(2016)He, Zhang, Ren, and Sun]{HeZRS16}
Kaiming He, Xiangyu Zhang, Shaoqing Ren, and Jian Sun.
\newblock Identity mappings in deep residual networks.
\newblock In \emph{European Conference on Computer Vision, {ECCV}}, volume 9908
  of \emph{Lecture Notes in Computer Science}, pages 630--645. Springer, 2016.

\bibitem[Hu et~al.(2020)Hu, Shen, Albanie, Sun, and Wu]{HuSASW20}
Jie Hu, Li~Shen, Samuel Albanie, Gang Sun, and Enhua Wu.
\newblock Squeeze-and-excitation networks.
\newblock \emph{{IEEE} Transactions on Pattern Analysis and Machine
  Intelligence}, 42\penalty0 (8):\penalty0 2011--2023, 2020.

\bibitem[Imamoglu et~al.(2018)Imamoglu, Oishi, Zhang, G.~Ding, Kouyama, and
  Nakamura]{Hyper2018}
N.~Imamoglu, Y.~Oishi, X.~Zhang, Y.~Fang G.~Ding, T.~Kouyama, and R.~Nakamura.
\newblock Hyperspectral image dataset for benchmarking on salient object
  detection.
\newblock In \emph{International Conference on Quality of Multimedia
  Experience, {QoME}}, pages 1--3, 2018.

\bibitem[Indyk et~al.(2019)Indyk, Vakilian, and Yuan]{IndykVY19}
Piotr Indyk, Ali Vakilian, and Yang Yuan.
\newblock Learning-based low-rank approximations.
\newblock In Hanna~M. Wallach, Hugo Larochelle, Alina Beygelzimer, Florence
  d'Alch{\'{e}}{-}Buc, Emily~B. Fox, and Roman Garnett, editors, \emph{Advances
  in Neural Information Processing Systems {NeurIPS}}, pages 7400--7410, 2019.

\bibitem[Jain et~al.(2020)Jain, Pillai, and Smith]{kac}
Vishesh Jain, Natesh Pillai, and Aaron Smith.
\newblock Kac meets johnson and lindenstrauss: a memory-optimal, fast
  johnson-lindenstrauss transform.
\newblock \emph{arXiv}, 03 2020.

\bibitem[Johnson and Lindenstrauss(1984)]{JohnsonL84}
William Johnson and Joram Lindenstrauss.
\newblock Extensions of lipschitz maps into a hilbert space.
\newblock \emph{Contemporary Mathematics}, 26:\penalty0 189--206, 01 1984.
\newblock \doi{10.1090/conm/026/737400}.

\bibitem[Kawaguchi(2016)]{Kawaguchi16}
Kenji Kawaguchi.
\newblock Deep learning without poor local minima.
\newblock In \emph{Advances in Neural Information Processing Systems
  {NeurIPS}}, pages 586--594, 2016.

\bibitem[Krahmer and Ward(2011)]{KrahmerW2011}
Felix Krahmer and Rachel Ward.
\newblock New and improved johnson–lindenstrauss embeddings via the
  restricted isometry property.
\newblock \emph{SIAM Journal on Mathematical Analysis}, 43:\penalty0
  1269--1281, 06 2011.
\newblock \doi{10.1137/100810447}.

\bibitem[Krizhevsky(2012)]{Krizhevsky09learningmultiple}
Alex Krizhevsky.
\newblock Learning multiple layers of features from tiny images.
\newblock \emph{University of Toronto}, 2012.

\bibitem[LeCun and Cortes(2010)]{MNIST2010}
Yann LeCun and Corinna Cortes.
\newblock {MNIST} handwritten digit database, 2010.

\bibitem[Lee et~al.(2019)Lee, Ajanthan, and Torr]{LeeAT19}
Namhoon Lee, Thalaiyasingam Ajanthan, and Philip H.~S. Torr.
\newblock Snip: single-shot network pruning based on connection sensitivity.
\newblock In \emph{International Conference on Learning Representations,
  {ICLR}}. OpenReview.net, 2019.

\bibitem[Li and Yang(2017)]{LiY17}
Yingzhou Li and Haizhao Yang.
\newblock Interpolative butterfly factorization.
\newblock \emph{{SIAM} Journal on Scientific Computing}, 39\penalty0 (2), 2017.

\bibitem[Li et~al.(2015)Li, Yang, Martin, Ho, and Ying]{LiYMHY15}
Yingzhou Li, Haizhao Yang, Eileen~R. Martin, Kenneth~L. Ho, and Lexing Ying.
\newblock Butterfly factorization.
\newblock \emph{Multiscale Model. Simul.}, 13\penalty0 (2):\penalty0 714--732,
  2015.

\bibitem[Li et~al.(2018)Li, Cheng, and Lu]{LiCL18}
Yingzhou Li, Xiuyuan Cheng, and Jianfeng Lu.
\newblock Butterfly-net: Optimal function representation based on convolutional
  neural networks.
\newblock \emph{CoRR}, abs/1805.07451, 2018.

\bibitem[Lu et~al.(2016)Lu, Sindhwani, and Sainath]{LuSS16}
Zhiyun Lu, Vikas Sindhwani, and Tara~N. Sainath.
\newblock Learning compact recurrent neural networks.
\newblock In \emph{International Conference on Acoustics, Speech and Signal
  Processing, {ICASSP}}, pages 5960--5964. {IEEE}, 2016.

\bibitem[Marcus et~al.(1993)Marcus, Santorini, and Marcinkiewicz]{treebank}
Mitchell~P. Marcus, Beatrice Santorini, and Mary~Ann Marcinkiewicz.
\newblock Building a large annotated corpus of {E}nglish: The {P}enn
  {T}reebank.
\newblock \emph{Computational Linguistics}, 19\penalty0 (2):\penalty0 313--330,
  1993.
\newblock URL \url{https://www.aclweb.org/anthology/J93-2004}.

\bibitem[{Michielssen} and {Boag}(1996)]{MichielssenB96}
E.~{Michielssen} and A.~{Boag}.
\newblock A multilevel matrix decomposition algorithm for analyzing scattering
  from large structures.
\newblock \emph{IEEE Transactions on Antennas and Propagation}, 44\penalty0
  (8):\penalty0 1086--1093, 1996.

\bibitem[Mocanu et~al.(2018)Mocanu, Mocanu, Stone, Nguyen, Gibescu, and
  Liotta]{MocanuESNG18}
Decebal~Constantin Mocanu, Elena Mocanu, Peter Stone, Phuong~H. Nguyen,
  Madeleine Gibescu, and Antonio Liotta.
\newblock Scalable training of artificial neural networks with adaptive sparse
  connectivity inspired by network science.
\newblock \emph{Nature Communications}, 9:\penalty0 2383, 2018.
\newblock \doi{10.1038/s41467-018-04316-3}.

\bibitem[Moczulski et~al.(2016)Moczulski, Denil, Appleyard, and
  de~Freitas]{MoczulskiDAF15}
Marcin Moczulski, Misha Denil, Jeremy Appleyard, and Nando de~Freitas.
\newblock {ACDC:} {A} structured efficient linear layer.
\newblock In Yoshua Bengio and Yann LeCun, editors, \emph{International
  Conference on Learning Representations, {ICLR}}, 2016.

\bibitem[O'Neil et~al.(2010)O'Neil, Woolfe, and Rokhlin]{NeilWR10}
Michael O'Neil, Franco Woolfe, and Vladimir Rokhlin.
\newblock An algorithm for the rapid evaluation of special function transforms.
\newblock \emph{Applied and Computational Harmonic Analysis}, 28\penalty0
  (2):\penalty0 203 -- 226, 2010.

\bibitem[Sainath et~al.(2013)Sainath, Kingsbury, Sindhwani, Arisoy, and
  Ramabhadran]{SainathKSAR13}
Tara~N. Sainath, Brian Kingsbury, Vikas Sindhwani, Ebru Arisoy, and Bhuvana
  Ramabhadran.
\newblock Low-rank matrix factorization for deep neural network training with
  high-dimensional output targets.
\newblock In \emph{International Conference on Acoustics, Speech and Signal
  Processing, {ICASSP}}, pages 6655--6659. {IEEE}, 2013.

\bibitem[Sarl{\'{o}}s(2006)]{Sarlos06}
Tam{\'{a}}s Sarl{\'{o}}s.
\newblock Improved approximation algorithms for large matrices via random
  projections.
\newblock In \emph{{IEEE} Symposium on Foundations of Computer Science {FOCS}},
  pages 143--152. {IEEE} Computer Society, 2006.

\bibitem[Seljebotn(2012)]{Seljebotn2012}
D.~S. Seljebotn.
\newblock {WAVEMOTH}-{FAST} {SPHERICAL} {HARMONIC} {TRANSFORMS} {BY}
  {BUTTERFLY} {MATRIX} {COMPRESSION}.
\newblock \emph{The Astrophysical Journal Supplement Series}, 199\penalty0
  (1):\penalty0 5, 2012.

\bibitem[Sindhwani et~al.(2015)Sindhwani, Sainath, and Kumar]{SindhwaniSK15}
Vikas Sindhwani, Tara~N. Sainath, and Sanjiv Kumar.
\newblock Structured transforms for small-footprint deep learning.
\newblock In Corinna Cortes, Neil~D. Lawrence, Daniel~D. Lee, Masashi Sugiyama,
  and Roman Garnett, editors, \emph{Advances in Neural Information Processing
  Systems {NeurIPS}}, pages 3088--3096, 2015.

\bibitem[Tan and Le(2019)]{TanL19}
Mingxing Tan and Quoc~V. Le.
\newblock Efficientnet: Rethinking model scaling for convolutional neural
  networks.
\newblock In Kamalika Chaudhuri and Ruslan Salakhutdinov, editors,
  \emph{International Conference on Machine Learning, {ICML}}, volume~97 of
  \emph{Proceedings of Machine Learning Research}, pages 6105--6114. {PMLR},
  2019.

\bibitem[Thomas et~al.(2018)Thomas, Gu, Dao, Rudra, and R{\'{e}}]{ThomasGDRR18}
Anna~T. Thomas, Albert Gu, Tri Dao, Atri Rudra, and Christopher R{\'{e}}.
\newblock Learning compressed transforms with low displacement rank.
\newblock In \emph{Advances in Neural Information Processing Systems
  {NeurIPS}}, pages 9066--9078, 2018.

\bibitem[Tjong Kim~Sang and De~Meulder(2003)]{CONLL}
Erik~F. Tjong Kim~Sang and Fien De~Meulder.
\newblock Introduction to the {C}o{NLL}-2003 shared task: Language-independent
  named entity recognition.
\newblock In \emph{Conference on Natural Language Learning at {HLT}-{NAACL}},
  pages 142--147, 2003.
\newblock URL \url{https://www.aclweb.org/anthology/W03-0419}.

\bibitem[Verdenius et~al.(2020)Verdenius, Stol, and Forr{\'{e}}]{VerdeniusSF20}
Stijn Verdenius, Maarten Stol, and Patrick Forr{\'{e}}.
\newblock Pruning via iterative ranking of sensitivity statistics.
\newblock \emph{CoRR}, abs/2006.00896, 2020.

\bibitem[Wang et~al.(2020)Wang, Zhang, and Grosse]{WangZG20}
Chaoqi Wang, Guodong Zhang, and Roger~B. Grosse.
\newblock Picking winning tickets before training by preserving gradient flow.
\newblock In \emph{International Conference on Learning Representations,
  {ICLR}}. OpenReview.net, 2020.

\bibitem[Yang et~al.(2015)Yang, Moczulski, Denil, de~Freitas, Smola, Song, and
  Wang]{YangMDFSSW15}
Zichao Yang, Marcin Moczulski, Misha Denil, Nando de~Freitas, Alexander~J.
  Smola, Le~Song, and Ziyu Wang.
\newblock Deep fried convnets.
\newblock In \emph{{IEEE} International Conference on Computer Vision, {ICCV}},
  pages 1476--1483. {IEEE} Computer Society, 2015.

\bibitem[Ye et~al.(2018)Ye, Wang, Li, Chen, Zhe, Chu, and Xu]{YeWLCZCX18}
Jinmian Ye, Linnan Wang, Guangxi Li, Di~Chen, Shandian Zhe, Xinqi Chu, and
  Zenglin Xu.
\newblock Learning compact recurrent neural networks with block-term tensor
  decomposition.
\newblock In \emph{Conference on Computer Vision and Pattern Recognition,
  {CVPR}}, pages 9378--9387. {IEEE} Computer Society, 2018.

\end{thebibliography}

\appendix
\section{Butterfly Diagram from Section \ref{SEC:INTRO}} \label{subsection bf plots}
Figure \ref{figbfnet} referred to in the introduction is given here.
\begin{figure}[!htb]
\centering
\minipage{0.45\textwidth}
  \includegraphics[width=\linewidth]{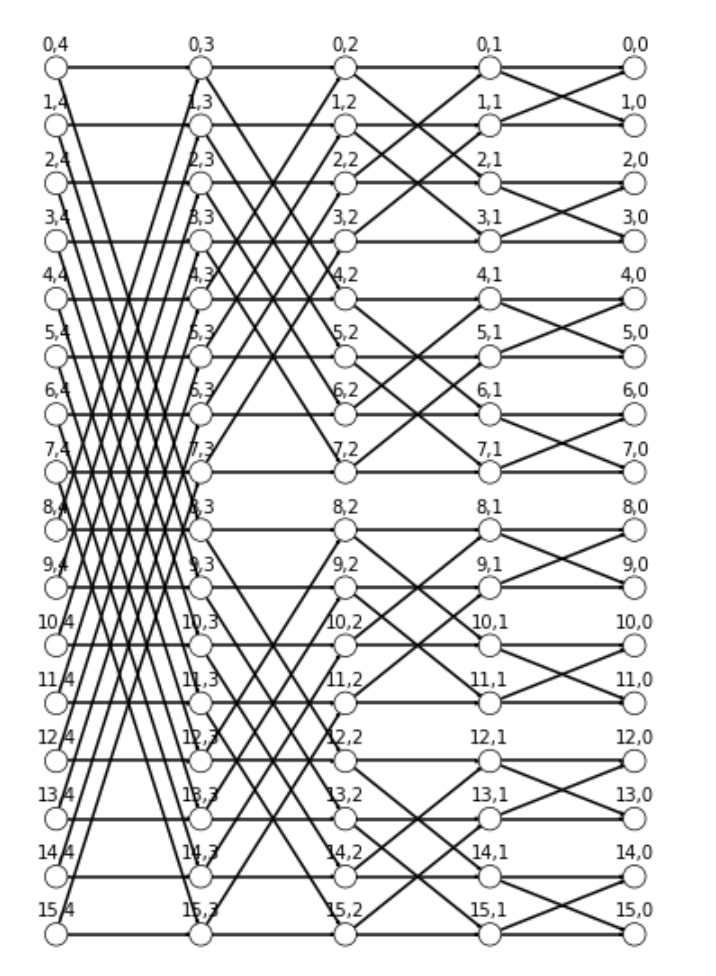}
\endminipage\hfill
\minipage{0.45\textwidth}
  \includegraphics[width=\linewidth]{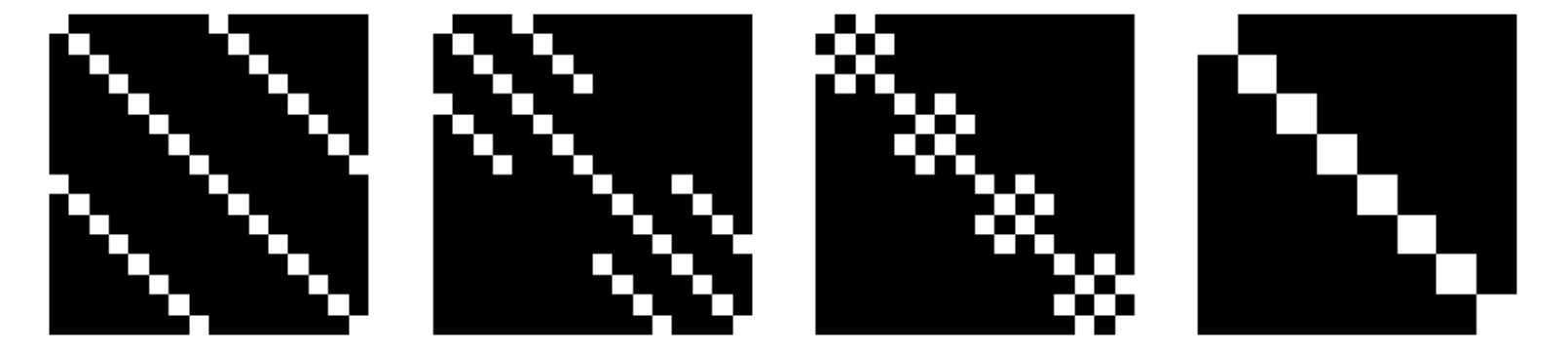}
\endminipage\hfill
 \caption{A $16 \times 16$ butterfly network represented as a $4$-layered graph on the left, and as product of $4$ sparse matrices on the right. The white entries are the non-zero entries of the matrices.}
  \label{figbfnet}
\end{figure}
\section{Proof of Proposition \ref{FACT:MATAPPFJLT}}\label{secappendix: proof of matrix approx theorem}
The proof of the proposition will use the following well known fact (Lemma \ref{nirlemma} below) about FJLT (more generally, JL) distributions (see \citet{,AilonC09,DBLP:journals/dcg/AilonL09,KrahmerW2011}).
\begin{lemma}\label{nirlemma}
Let $\mathbf{x}\in \R^n$ be a unit vector, and let $J\in \R^{k\times n}$ be a matrix drawn from an FJLT distribution.  Then for all $\epsilon<1$ with probability at least $1-e^{-\Omega(k\epsilon^2)}$: 
\begin{equation}\label{nirproof0}\|\mathbf{x} - J^TJ\mathbf{x}\| \leq \epsilon\ .
\end{equation}
\end{lemma}
By Lemma~\ref{nirlemma} we have that with probability at least $1-e^{-\Omega( k_1 \epsilon^2)}$,
\begin{equation}\label{nirproof1}
\|\mathbf{x} - J_1^T J_1 \mathbf{x}\| \leq \epsilon\|\mathbf{x}\| = \epsilon \ .
\end{equation}
Henceforth, we condition on the event $\|\mathbf{x} - J_1^T J_1 \mathbf{x}\| \leq \epsilon \|\mathbf{x}\|$. 
Therefore, by the definition of spectral norm $\|W\|$ of $W$:
\begin{equation}\label{nirproof2}
\|W\mathbf{x} - WJ_1^T J_1 \mathbf{x}\| \leq \epsilon\|W\| \ .
\end{equation}
Now apply Lemma~\ref{nirlemma} again on the vector $ WJ_1^T J_1 \mathbf{x}$ and transformation $J_2$ to get that with probability at least $1-e^{-\Omega(k_2\epsilon^2)}$, \begin{equation}\label{nirproof4}\|WJ_1^T J_1 \mathbf{x}-J_2^T J_2 WJ_1^T J_1 \mathbf{x}\| \leq \epsilon \|WJ_1^T J_1 \mathbf{x}\|.
\end{equation}
Henceforth, we condition on the event $\|WJ_1^T J_1 \mathbf{x}-J_2^T J_2 WJ_1^T J_1 \mathbf{x}\| \leq \epsilon \|WJ_1^T J_1 \mathbf{x}\|$.
To bound the last right hand side, we use the triangle inequality together with (\ref{nirproof2}):
\begin{equation}\label{nirproof5}
\|WJ_1^TJ_1\mathbf{x}\| \leq \|W\mathbf{x}\| + \epsilon \|W\| \leq \|W\|(1+\epsilon).
\end{equation}
Combining (\ref{nirproof4}) and (\ref{nirproof5}) gives:
\begin{equation}\label{nirproof6}\|WJ_1^T J_1 \mathbf{x}-J_2^T J_2 WJ_1^T J_1 \mathbf{x}\| \leq \epsilon \|W\|(1+\epsilon).
\end{equation}
Finally,
\begin{eqnarray}
\|J_2^T J_2 WJ_1^T J_1 \mathbf{x} - W\mathbf{x}\| &=& \|(J_2^T J_2 WJ_1^T J_1 \mathbf{x} - WJ_1^T J_1 \mathbf{x}) + (WJ_1^T J_1 \mathbf{x}  - W\mathbf{x})\| \nonumber \\
&\leq & \epsilon \|W\|(1+\epsilon) + \epsilon \|W\| \nonumber \\
&=& \|W\|\epsilon(2+\epsilon) \leq 3\|W\|\epsilon\ ,
\end{eqnarray}
where the first inequality is from the triangle inequality together with (\ref{nirproof2}) and (\ref{nirproof6}), and the second inequality is from the bound on $\epsilon$.
The proposition is obtained by adjusting the constants hiding inside the $\Omega()$ notation in the exponent in the proposition statement.
\section{Proof of Theorem \ref{THM:CP}}\label{secappendix: proof of theorem}
We first note that our result continues to hold even if $B$ in the theorem is replaced by any structured matrix. For example the result continues to hold if $B$ is an $\ell \times n$ matrix with one non-zero entry per column, as is the case with a random sparse sketching matrix \cite{ClarksonW09}. We also compare our result with that \cite{BaldiH89,Kawaguchi16}.

\textbf{Comparison with \cite{BaldiH89} and \cite{Kawaguchi16}}: The critical points of the encoder-decoder network are analyzed in \cite{BaldiH89}. Suppose the eigenvalues of $YX^T(XX^T)^{-1}XY^T$ are $\gamma_1 > \ldots > \gamma_{m}>0$ and $k\leq m \leq n$. Then they show that corresponding to a critical point there is an $I\subseteq [m]$ such that the loss at this critical point is equal to $\text{tr}(YY^T) - \sum_{i\in I} \gamma_i$, and the critical point is a local/global minima if and only if $I=[k]$. \cite{Kawaguchi16} later generalized this to prove that a local minima is a global minima for an arbitrary number of hidden layers in a linear neural network if $m\leq n$. Note that since $\ell \leq n$ and $m \leq n$ in Theorem \ref{THM:CP}, replacing $X$ by $BX$ in \cite{BaldiH89} or \cite{Kawaguchi16} does not imply Theorem \ref{THM:CP} as it is.

Next, we introduce a few notation before delving into the proof. Let $r = (\overline{Y}-Y)^T$, and $\text{vec}(r) \in \R^{md}$ is the entries of $r$ arranged as a vector in column-first ordering, $(\nabla_{\textnormal{vec}(D^T)}\mathcal{L}(\overline{Y}))^T \in \R^{mk}$ and $(\nabla_{\textnormal{vec}(E^T)}\mathcal{L}(\overline{Y}))^T \in \R^{k\ell}$  denote the partial derivative of $\mathcal{L}(\overline{Y})$ with respect to the parameters in $\text{vec}(D^T)$ and $\text{vec}(E^T)$ respectively. Notice that $\nabla_{\textnormal{vec}(D^T)}\mathcal{L}(\overline{Y})$ and $\nabla_{\textnormal{vec}(E^T)}\mathcal{L}(\overline{Y})$ are row vectors of size $mk$ and $k\ell$ respectively. Also, let $P_D$ denote the projection matrix of $D$, and hence if $D$ is a matrix with full column-rank then $P_D = D(D^T\cdot D)^{-1}\cdot D^T$. The $n\times n$ identity matrix is denoted as $I_{n}$, and for convenience of notation let $\tilde{X} = B\cdot X$. First we prove the following lemma which gives an expression for $D$ and $E$ if  $\nabla_{\textnormal{vec}(D^T)}\mathcal{L}(\overline{Y})$ and $\nabla_{\textnormal{vec}(E^T)}\mathcal{L}(\overline{Y})$ are zero.
\begin{lemma}[Derivatives with respect to $D$ and $E$] \label{lemma: derivative}
\begin{enumerate}
    \item[]
    \item $\nabla_{\textnormal{vec}(D^T)} \mathcal{L}(\overline{Y}) = \text{vec}(r)^T (I_{m} \otimes (E\cdot \tilde{X})^T)$, and
    \item $\nabla_{\textnormal{vec}(E^T)} \mathcal{L}(\overline{X}) = \text{vec}(r)^T (D \otimes \tilde{X})^T $
\end{enumerate}
\end{lemma}
\begin{proof}
\begin{enumerate}
    \item Since $\mathcal{L}(\overline{Y}) = \frac{1}{2}\text{vec}(r)^T\cdot \text{vec}(r)$,
\begin{align*}
    \nabla_{\text{vec}(D^T)} \mathcal{L}(\overline{Y}) &= \text{vec}(r)^T \cdot \nabla_{\text{vec}(D^T)} \text{vec}(r) 
    ~= ~\text{vec}(r)^T (\text{vec}_{(D^T)}(\tilde{X}^T\cdot E^T \cdot D^T) )\\
     &= \text{vec}(r)^T (I_{m} \otimes (E\cdot \tilde{X})^T  )\cdot \nabla_{\text{vec}(D^T)} \text{vec}(D^T) 
     ~~~=~~~ \text{vec}(r)^T (I_{m} \otimes (E\cdot \tilde{X})^T)
\end{align*}
 \item Similarly,
 \begin{align*}
\nabla_{\text{vec}(E^T)} \mathcal{L}(\overline{Y}) &= \text{vec}(r)^T \cdot \nabla_{\text{vec}(E^T)} \text{vec}(r) 
    ~=~ \text{vec}(r)^T (\text{vec}_{(E^T)}(\tilde{X}^T\cdot E^T \cdot D^T) )\\
     &= \text{vec}(r)^T (D \otimes \tilde{X}^T )\cdot \nabla_{\text{vec}(E^T)} \text{vec}(E^T) 
     ~=~\text{vec}(r)^T (D \otimes \tilde{X}^T)
\end{align*}
\end{enumerate}
\end{proof}
Assume the rank of $D$ is equal to $p$. Hence there is an invertible matrix $C \in \R^{k\times k}$ such that $\tilde{D} = D\cdot C$ is such that the last $k - p$ columns of $\tilde{D}$ are zero and the first $p$ columns of $\tilde{D}$ are linearly independent (via Gauss elimination). Let $\tilde{E} = C^{-1}\cdot E$. Without loss of generality it can be assumed $\tilde{D} \in \R^{d\times p}$, and $\tilde{E}\in \R^{p\times d}$, by restricting restricting $\tilde{D}$ to its first $p$ columns (as the remaining are zero) and $\tilde{E}$ to its first $p$ rows. Hence, $\tilde{D}$ is a full column-rank matrix of rank $p$, and $D E = \tilde{D} \tilde{E}$. Claims \ref{claim: representation at the critical point} and \ref{claim: commutativity of P_D with sum} aid us in the completing the proof of the theorem. First the proof of theorem is completed using these claims, and at the end the two claims are proved.
\begin{claim}[Representation at the critical point]\label{claim: representation at the critical point}
\begin{enumerate}
    \item[]
    \item $\tilde{E} = (\tilde{D}^T \tilde{D})^{-1} \tilde{D}^T  Y  \tilde{X}^T (\tilde{X} \cdot \tilde{X}^T)^{-1} $
    \item $\tilde{D} \tilde{E} = P_{\tilde{D}}  Y \tilde{X}^T  (\tilde{X} \cdot \tilde{X}^T)^{-1}$
\end{enumerate}
\end{claim}
\begin{claim}\label{claim: commutativity of P_D with sum}
\begin{enumerate}
    \item $\tilde{E} B \tilde{D} = (\tilde{E}B Y \tilde{X}^T \tilde{E}^T)(\tilde{E} \tilde{X} \tilde{X}^T \tilde{E}^T)^{-1} $
    \item $P_{\tilde{D}} \Sigma = \Sigma P_{\tilde{D}} = P_{\tilde{D}} \Sigma  P_{\tilde{D}}$
\end{enumerate}
\end{claim}
We denote $\Sigma(B)$ as $\Sigma$ for convenience. Since $\Sigma$ is a real symmetric matrix, there is an orthogonal matrix $U$ consisting of the eigenvectors of $\Sigma$, such that $\Sigma = U\wedge U^T$, where $\wedge$ is a $m\times m$ diagonal matrix whose first $\ell$ diagonal entries are $\lambda_1, \ldots, \lambda_\ell$ and the remaining entries are zero. Let $u_1, \ldots, u_{m}$ be the columns of $U$. Then for $i\in [\ell]$, $u_i$ is the eigenvector of $\Sigma$ corresponding to the eigenvalue $\lambda_i$, and $\{u_{\ell+1}, \ldots , u_{d_y}\}$ are the eigenvectors of $\Sigma$ corresponding to the eigenvalue $0$. \\
\\
Note that $P_{U^T\tilde{D}} = U^T\tilde{D}(\tilde{D}^TU^TU \tilde{D})^{-1} \tilde{D}^TU = U^T P_{\tilde{D}} U$, and from part two of Claim \ref{claim: commutativity of P_D with sum} we have
\begin{align}
    (U P_{U^T\tilde{D}} U^T) \Sigma &= \Sigma (U P_{U^T\tilde{D}} U^T) \\
    U\cdot P_{U^T\tilde{D}} \wedge U^T & = U \wedge P_{U^T\tilde{D}} U^T \\
    P_{U^T\tilde{D}} \wedge &= \wedge P_{U^T\tilde{D}}
\end{align}
Since $P_{U^T\tilde{D}}$ commutes with $\wedge$, $P_{U^T\tilde{D}}$ is a block-diagonal matrix comprising of two blocks $P_1$ and $P_2$: the first block $P_1$ is an $\ell \times \ell$ diagonal block, and $P_2$ is a $(m-\ell) \times (m-\ell)$ matrix. Since $P_{U^T\tilde{D}}$ is orthogonal projection matrix of rank $p$ its eigenvalues are $1$ with multiplicity $p$ and $0$ with multiplicity $m-p$. Hence at most $p$ diagonal entries of $P_1$ are $1$ and the remaining are $0$. Finally observe that
\begin{align*}
    \mathcal{L}(\overline{Y}) &= \text{tr}((\overline{Y}-Y)(\overline{Y}-Y)^T)\\
    &= \text{tr}(YY^T) - 2\text{tr}(\overline{Y}Y^T) + \text{tr}(\overline{Y}\overline{Y}^T)\\ 
    &= \text{tr}(YY^T) - 2\text{tr}(P_{\tilde{D}}\Sigma) + \text{tr}(P_{\tilde{D}} \Sigma P_{\tilde{D}}) \\
    & = \text{tr}(YY^T) - \text{tr}(P_{\tilde{D}}\Sigma)
\end{align*}
The second line in the above equation follows using the fact that $\text{tr}(\overline{Y}Y^T) = \text{tr}(Y\overline{Y}^T)$, the third line in the above equation follows by substituting $\overline{Y} = P_{\tilde{D}}Y \tilde{X}^T \cdot (\tilde{X} \cdot \tilde{X}^T)^{-1}\cdot \tilde{X}$ (from part two of Claim \ref{claim: representation at the critical point}), and the last line follows from part two of Claim \ref{claim: commutativity of P_D with sum}. Substituting $\Sigma = U\wedge U^T$, and $P_{\tilde{D}} = UP_{U^T\tilde{D}}U^T$ in the above equation we have,
\begin{align*}
    \mathcal{L}(\overline{Y}) &= \text{tr}(YY^T) - \text{tr}(UP_{U^T\tilde{D}}\wedge U^T) \\
    &= \text{tr}(YY^T) - \text{tr}(P_{U^T\tilde{D}}\wedge)
\end{align*}
The last line the above equation follows from the fact that $\text{tr}(UP_{\tilde{U^TD}}\wedge U^T) = \text{tr}(P_{U^T\tilde{D}}\wedge U^TU) = \text{tr}(P_{U^T\tilde{D}}\wedge)$. From the structure of $P_{U^T\tilde{D}}$ and $\wedge$ it follows that there is a subset $I\subseteq [\ell]$, $|I| \leq p$ such that $\text{tr}(P_{U^T\tilde{D}}\wedge) = \sum_{i\in I}\lambda_i$. Hence, $\mathcal{L}(\overline{Y}) = \text{tr}(YY^T) - \sum_{i\in I}\lambda_i$. \\
\\
Since $P_{\tilde{D}} = UP_{U^T\tilde{D}}U^T$, there is a $p\times p$ invertible matrix $M$ such that  
$$\tilde{D} = (U\cdot V)_{I'}\cdot M~~, \text{and}~~~ \tilde{E} = M^{-1}(V^TU^T)_{I'}Y \tilde{X}^T(\tilde{X}\tilde{X}^T)^{-1}$$ 
where $V$ is a block-diagonal matrix consisting of two blocks $V_1$ and $V_2$: $V_1$ is equal to $I_{\ell}$, and $V_2$ is an $(m-\ell) \times (m -\ell)$ orthogonal matrix, and $I'$ is such that $I\subseteq I'$ and $|I'| = p$. The relation for $\tilde{E}$ in the above equation follows from part one of Claim \ref{claim: representation at the critical point}. Note that if $I' \subseteq [\ell]$, then $I = I'$, that is $I$ consists of indices corresponding to eigenvectors of non-zero eigenvalues.  \\
\\
Recall that $\tilde{D}$ was obtained by truncating the last $k-p$ zero rows of $DC$, where $C$ was a $k\times k$ invertible matrix simulating the Gaussian elimination. Let $[M|O_{p\times (k-p)}]$ denoted the $p\times k$ matrix obtained by augmenting the columns of $M$ with $(k-p)$ zero columns. Then 
$$D = (UV)_{I'}[M|O_{p\times (k-p)}]C^{-1} ~.$$ 
Similarly, there is a $p\times (k-p)$ matrix $N$ such that 
$$E = C[{ \scriptstyle \frac{M^{-1}}{N}}]((UV)_{I'})^T Y \tilde{X}^T(\tilde{X}\tilde{X}^T)^{-1}~$$
where $[\frac{M^{-1}}{N}]$ denotes the $k\times p$ matrix obtained by augmenting the rows of $M^{-1}$ with the rows of $N$. 
Now suppose $I \neq [k]$, and hence $I' \neq [k]$. Then we will show that there are matrices $D'$ and $E'$ arbitrarily close to $D$ and $E$ respectively such that if $Y' = D'E'\tilde{X}$ then $\mathcal{L}(Y') < \mathcal{L}(\overline{Y})$. There is an $a\in [k]\setminus I'$, and $b\in I'$ such that $\lambda_a > \lambda_b$ ($\lambda_b$ could also be zero). Denote the columns of the matrix $UV$ as $\{v_1, \ldots , v_{m}\}$, and observe that $v_i = u_i$ for $i\in [\ell]$ (from the structure of $V$). For $\epsilon >0$ let $u'_b = (1+\epsilon^2)^{-\frac{1}{2}}(v_b + \epsilon u_a)$. Define $U'$ as the matrix which is equal to $UV$ except that the column vector $v_b$ in $UV$ is replaced by $u'_b$ in $U'$. Since $a\in [k] \subseteq [\ell]$ and $a\notin I'$, $v_a = u_a$ and $(U'_{I'})^T U'_{I'} = I_p$. Define 
$$D' = U'_{I'}[M|O_{p\times (k-p)}]C^{-1}~~, \text{and}~~~  E' = C[{ \scriptstyle \frac{M^{-1}}{N}}](U'_{I'})^T Y \tilde{X}^T(\tilde{X}\tilde{X}^T)^{-1}~$$   
and let $Y' = D'E'\tilde{X}$. Now observe that, $D'E' = U'_{I'}(U_{I'})^T Y \tilde{X}^T(\tilde{X}\tilde{X}^T)^{-1}$, and that 
$$\mathcal{L}(Y') = \text{tr}(YY^T) - \sum_{i\in I}\lambda_i - \frac{\epsilon^2}{1+\epsilon^2}(\lambda_a -\lambda_b) = \mathcal{L}(\overline{Y}) - \frac{\epsilon^2}{1+\epsilon^2}(\lambda_a -\lambda_b)$$
Since $\epsilon$ can be set arbitrarily close to zero, it can be concluded that there are points in the neighbourhood of $\overline{Y}$ such that the loss at these points are less than $\mathcal{L}(\overline{Y})$. Further, since $\mathcal{L}$ is convex with respect to the parameters in $D$ (respectively $E$), when the matrix $E$ is fixed (respectively $D$ is fixed) $\overline{Y}$ is not a local maximum. Hence, if $I \neq [k]$ then $\overline{Y}$ represents a saddle point, and in particular $\overline{Y}$ is local/global minima if and only if $I = [k]$.

\begin{proof}[Proof of Claim \ref{claim: representation at the critical point}]
Since $\nabla_{\textnormal{vec}(E^T)} \mathcal{L}(\overline{X})$ is equal to zero, from the second part of Lemma \ref{lemma: derivative} the following holds, 
\begin{align*}
    \tilde{X} (Y - \overline{Y})^T D &= \tilde{X} Y^T  D -  \tilde{X}\overline{Y}^T D = 0\\
  \Rightarrow \tilde{X}  \tilde{X}^T E^T D^T D   &=  \tilde{X} Y^T D
\end{align*}
Taking transpose on both sides
\begin{align}\label{equation: derivative at critical point with respect to E}
  \Rightarrow D^T D E \tilde{X} \tilde{X}^T &= D^T Y \tilde{X}^T 
\end{align}
Substituting $DE$ as $\tilde{D}\tilde{E}$ in Equation \ref{equation: derivative at critical point with respect to E}, and multiplying Equation \ref{equation: derivative at critical point with respect to E} by $C^{T}$ on both the sides from the left, Equation \ref{equation: derivative at critical point with respect to E after rank manipulation} follows.
\begin{align}\label{equation: derivative at critical point with respect to E after rank manipulation}
  \Rightarrow \tilde{D}^T \tilde{D} \tilde{E} \tilde{X} \tilde{X}^T &= \tilde{D}^T Y \tilde{X}^T 
\end{align}
Since $\tilde{D}$ is full-rank, we have
\begin{equation}\label{equation: for E at the critical point after rank manipulation}
   \tilde{E} = (\tilde{D}^T \tilde{D})^{-1} \tilde{D}^T  Y  \tilde{X}^T (\tilde{X} \tilde{X}^T)^{-1} . 
\end{equation}
and, 
\begin{equation}\label{equation: expression of decoder encoder at the critical point}
    \tilde{D} \tilde{E} = P_{\tilde{D}}  Y \tilde{X}^T  (\tilde{X}  \tilde{X}^T)^{-1}
\end{equation}
\end{proof}

\begin{proof}[Proof of Claim \ref{claim: commutativity of P_D with sum}]
Since $\nabla_{\textnormal{vec}(D^T)} \mathcal{L}(\overline{Y})$ is zero, from the first part of Lemma \ref{lemma: derivative} the following holds,
\begin{equation*}
    E\tilde{X} (Y - \overline{Y})^T = E\tilde{X} Y^T -  E \tilde{X}\cdot\overline{Y}^T = 0
\end{equation*}
\begin{equation}\label{equation: derivative at critical point with respect to D}
  \Rightarrow   E\tilde{X} \tilde{X}^T E^T D^T   =  E \tilde{X} Y^T  
\end{equation}
Substituting $E^T\cdot D^T$ as $\tilde{E}^T\cdot \tilde{D}^T$ in Equation \ref{equation: derivative at critical point with respect to E}, and multiplying Equation \ref{equation: derivative at critical point with respect to D} by $C^{-1}$ on both the sides from the left Equation \ref{equation: derivative at critical point with respect to D after rank manipulation} follows.
\begin{equation}\label{equation: derivative at critical point with respect to D after rank manipulation}
    \tilde{E}\tilde{X} \tilde{X}^T \tilde{E}^T \tilde{D}^T   = \tilde{E} \tilde{X} Y^T  
\end{equation}
Taking transpose of the above equation we have,
\begin{equation}\label{equation: derivative at critical point with respect to D after rank manipulation and transpose}
   \tilde{D}\tilde{E} \tilde{X} \tilde{X}^T \tilde{E}^T = Y  \tilde{X}^T \tilde{E}^T 
\end{equation}
From part 1 of Claim \ref{claim: representation at the critical point}, it follows that $\tilde{E}$ has full row-rank, and hence $\tilde{E} \tilde{X} \tilde{X}^T\tilde{E}^T$ is invertible. Multiplying the inverse of $\tilde{E} \tilde{X} \tilde{X}^T \tilde{E}^T$ from the right on both sides and multiplying $\tilde{E}B$ from the left on both sides of the above equation we have,
\begin{equation}
    \tilde{E}B \tilde{D} = (\tilde{E} B Y \tilde{X}^T \tilde{E}^T)(\tilde{E} \tilde{X} \tilde{X}^T \tilde{E}^T)^{-1} 
\end{equation}
This proves part one of the claim. Moreover, multiplying Equation \ref{equation: derivative at critical point with respect to D after rank manipulation and transpose} by $\tilde{D}^T$ from the right on both sides
\begin{align*}
   \tilde{D} \tilde{E} \tilde{X} \tilde{X}^T \tilde{E}^T \tilde{D}^T &= Y \tilde{X}^T \tilde{E}^T \tilde{D}^T  \\
 \Rightarrow   (P_{\tilde{D}} Y \tilde{X}^T (\tilde{X} \tilde{X}^T)^{-1}) ( \tilde{X}  \tilde{X}^T)  ((\tilde{X}  \tilde{X}^T)^{-1}  \tilde{X} Y^T P_{\tilde{D}}) &= Y  \tilde{X}^T ((\tilde{X} \tilde{X}^T)^{-1}  \tilde{X} Y^T\cdot P_{\tilde{D}})\\
  \Rightarrow  P_{\tilde{D}} Y \tilde{X}^T (\tilde{X}  \tilde{X}^T)^{-1}  \tilde{X} Y^T P_{\tilde{D}} &= Y  \tilde{X}^T (\tilde{X} \tilde{X}^T)^{-1}  \tilde{X} Y^T\cdot P_{\tilde{D}}
\end{align*}
The second line the above equation follows by  substituting $\tilde{D}\tilde{E} = P_{\tilde{D}} Y \tilde{X}^T (\tilde{X} \tilde{X}^T)^{-1}$ (from part 2 of Claim \ref{claim: representation at the critical point}). Substituting $\Sigma = Y \tilde{X}^T (\tilde{X} \tilde{X}^T)^{-1} \tilde{X} Y^T$ in the above equation we have
\begin{equation*}
    P_{\tilde{D}} \Sigma  P_{\tilde{D}} = \Sigma \cdot P_{\tilde{D}}
\end{equation*}
Since $P_{\tilde{D}}^T = P_{\tilde{D}}$, and $\Sigma^{T}=\Sigma$, we also have $\Sigma P_{\tilde{D}} = P_{\tilde{D}} \Sigma$.
\end{proof}

\section{Additional Tables and Plots related to Dense Layer Replacement}

\subsection{Plots from Section \ref{EXP:DENSE}}\label{subsecappendix: plots from dense}
Figure \ref{figurenumberofparams complete model} displays the number of parameter in the original model and the butterfly model.  Figure \ref{figureNLP} reports the results for the NLP tasks done as part of experiment in Section \ref{EXP:DENSE}. Figures \ref{figuretrainingtime} and \ref{figuretrainingtime for nlp} reports the training and inference times required for the original model and the butterfly model in each of the experiments. The training and and inference times in Figures \ref{figuretrainingtime} and \ref{figuretrainingtime for nlp} are averaged over 100 runs. Figure \ref{figure20epochs} is the same as the right part of Figure \ref{figureimgclass} but here we compare the test accuracy of the original and butterfly model for the the first 20 epochs.

\begin{figure}[!htb]
\minipage{0.45\textwidth}
  \includegraphics[width=\linewidth]{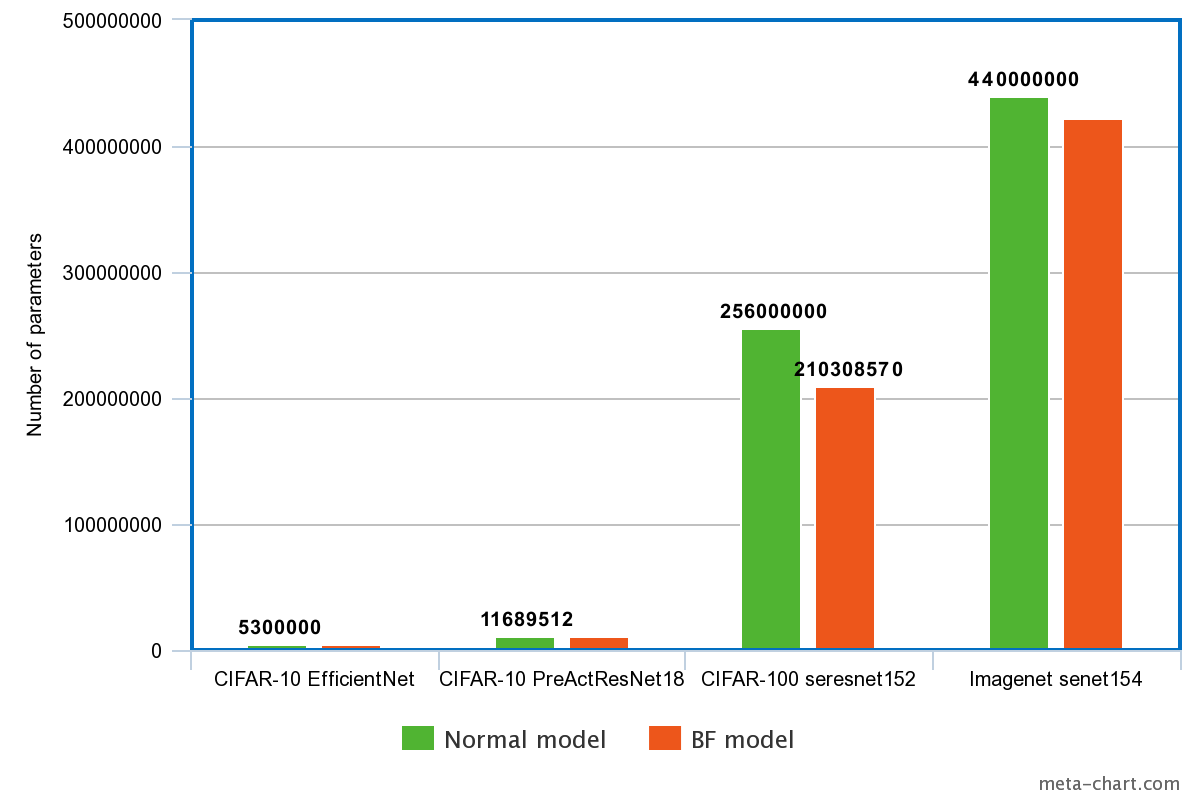}
  \endminipage\hfill
  \minipage{0.45\textwidth}
  \includegraphics[width=\linewidth]{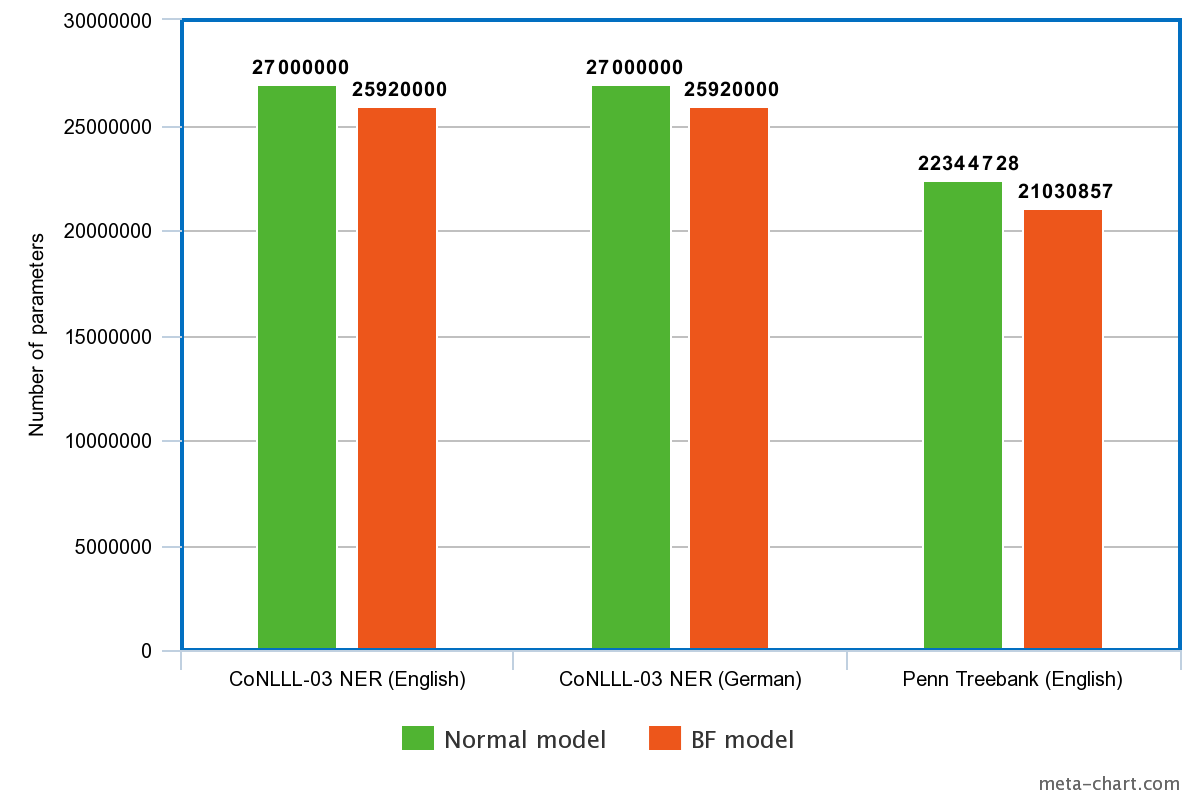}
\endminipage\hfill
\caption{Total number of parameters in the original model and the butterfly model; Left: Vision data, Right: NLP}
  \label{figurenumberofparams complete model}
\end{figure}

\begin{figure}[!htb]
\minipage{0.45\textwidth}
  \includegraphics[width=\linewidth]{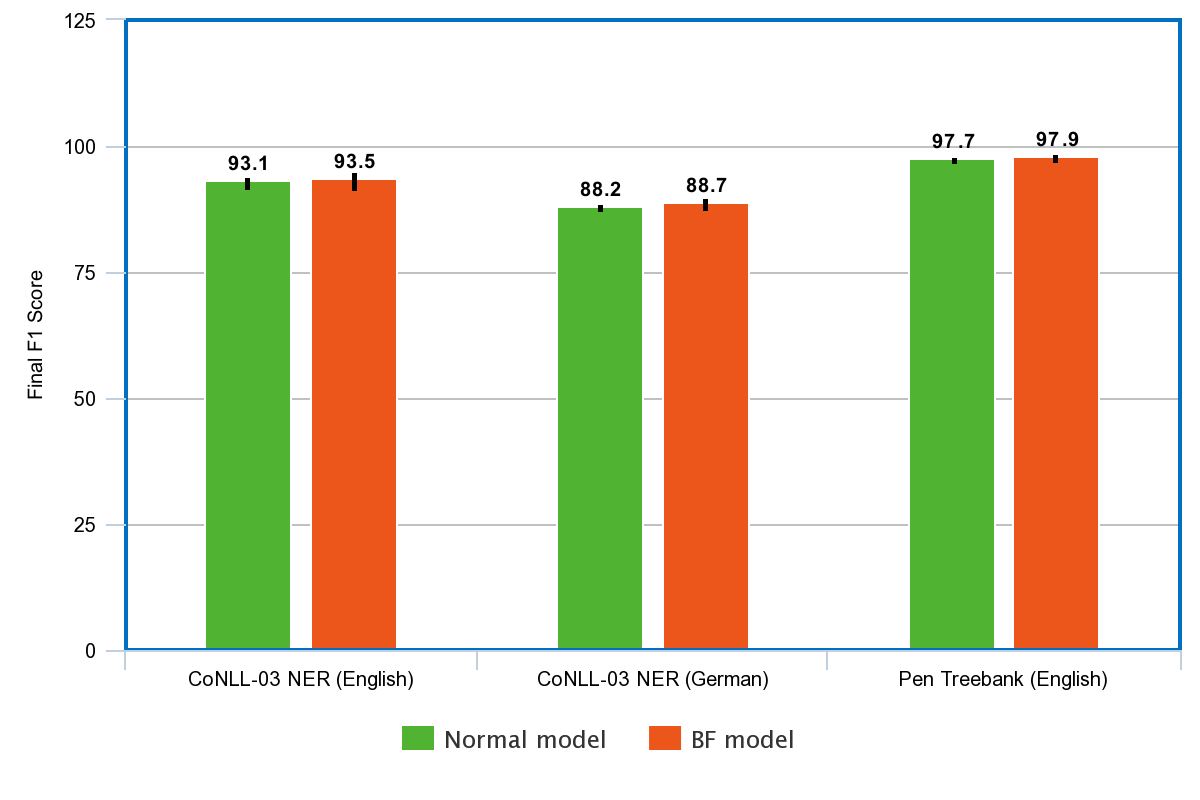}
  \label{figurefirstepochs}
\endminipage\hfill
\minipage{0.45\textwidth}
  \includegraphics[width=\linewidth]{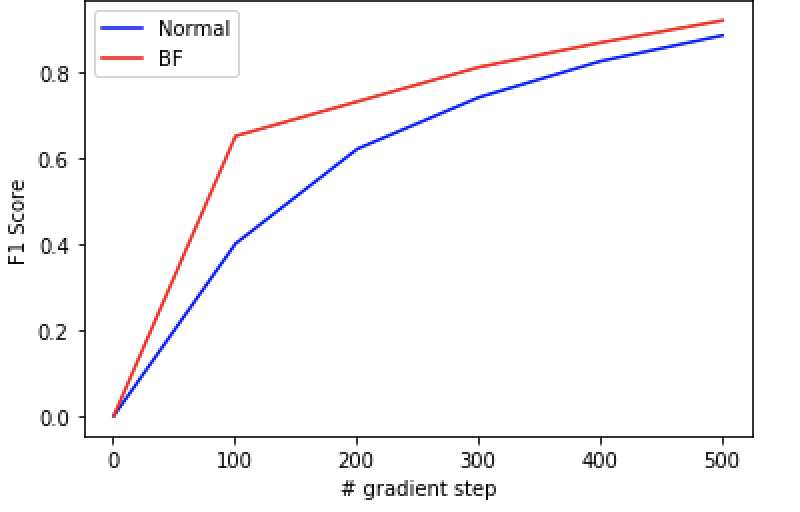}
  \label{figurefirstepochsNLP}
\endminipage\hfill
 \caption{Left: F1 comparison in the first few epochs with different models on CoNLL-03 Named Entity Recognition (English) with the flair's Sequence Tagger, Right: Final F1 Score for different NLP models and data sets. }
 \label{figureNLP}
\end{figure}


\begin{figure}[!htb]
\minipage{0.45\textwidth}
  \includegraphics[width=\linewidth]{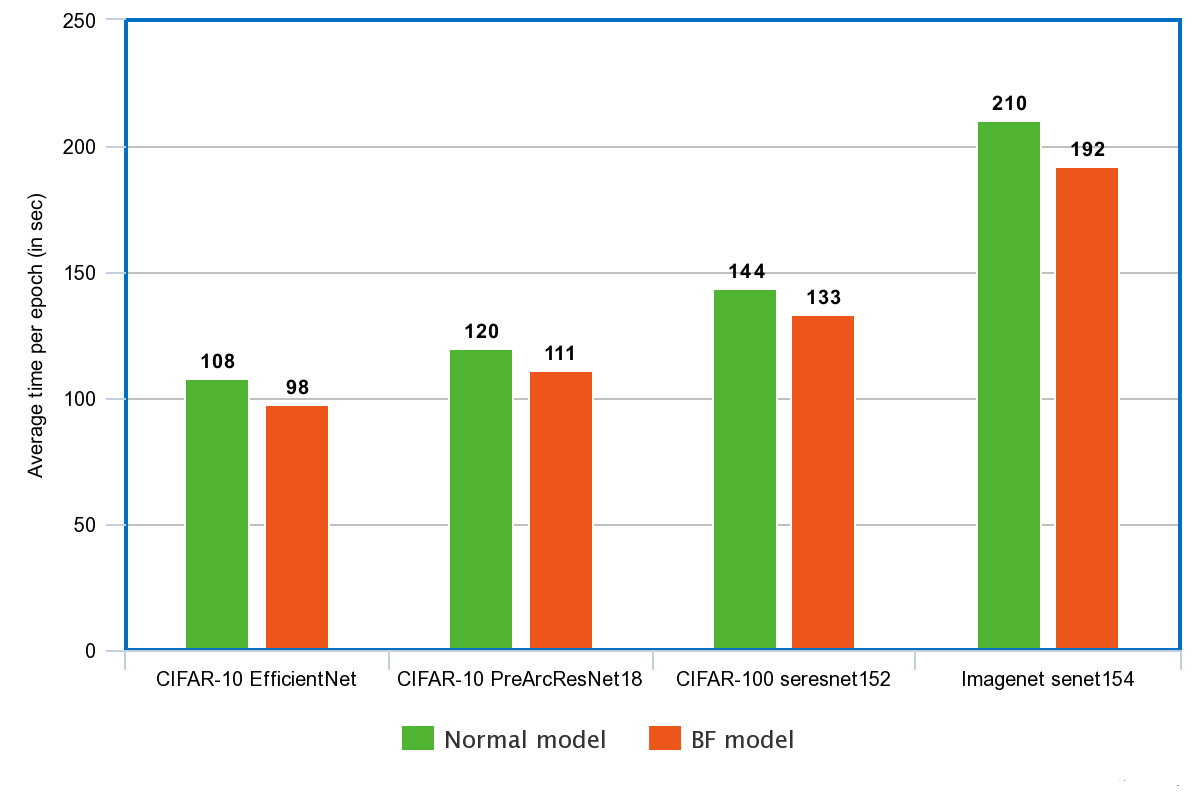}
  \endminipage\hfill
  \minipage{0.45\textwidth}
  \includegraphics[width=\linewidth]{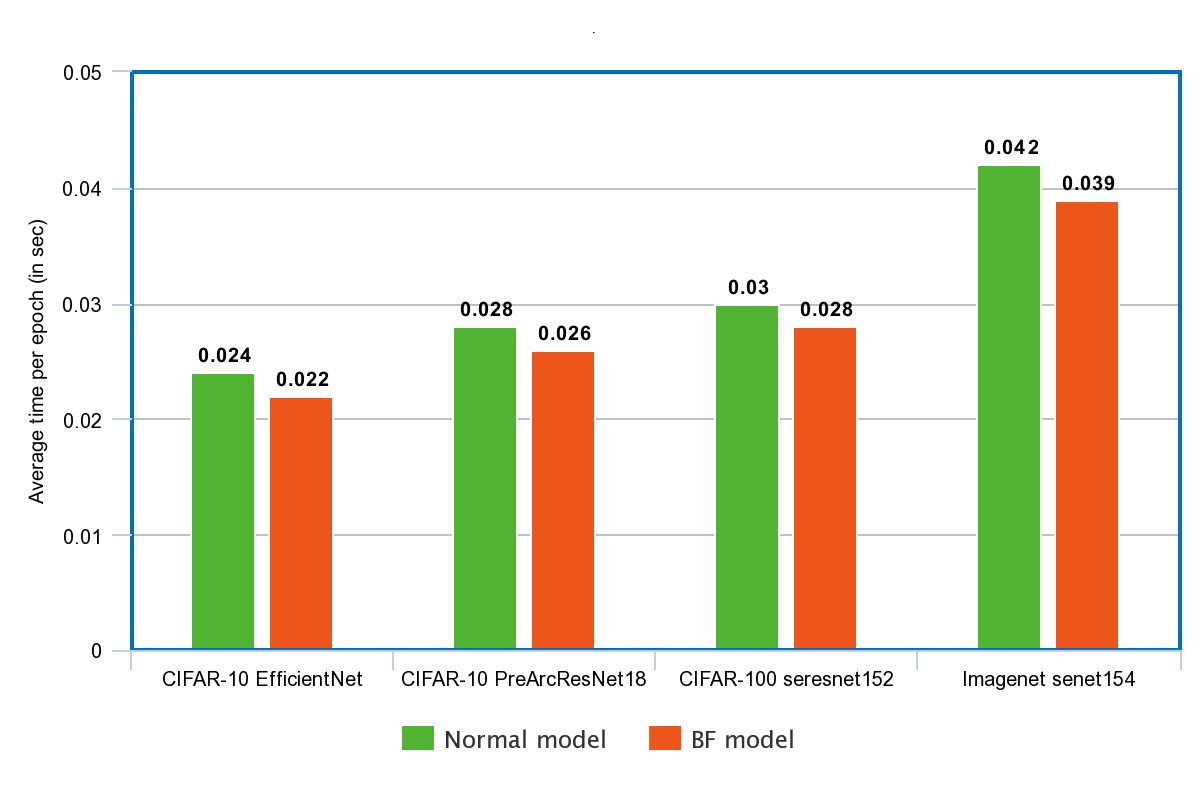}
\endminipage\hfill
\caption{Training/Inference times for Vision Data; Left: Training time, Right: Inference time}
  \label{figuretrainingtime}
\end{figure}

\begin{figure}[!htb]
\minipage{0.45\textwidth}
  \includegraphics[width=\linewidth]{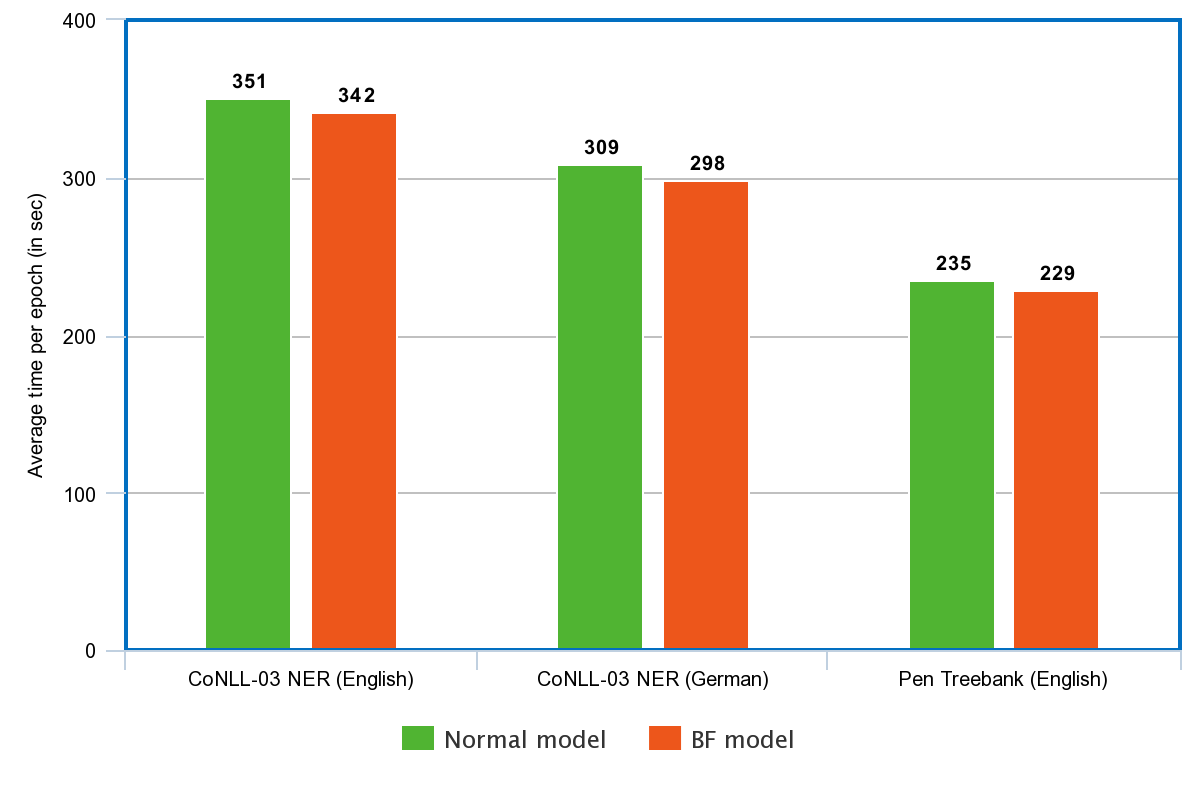}
  \endminipage\hfill
  \minipage{0.45\textwidth}
  \includegraphics[width=\linewidth]{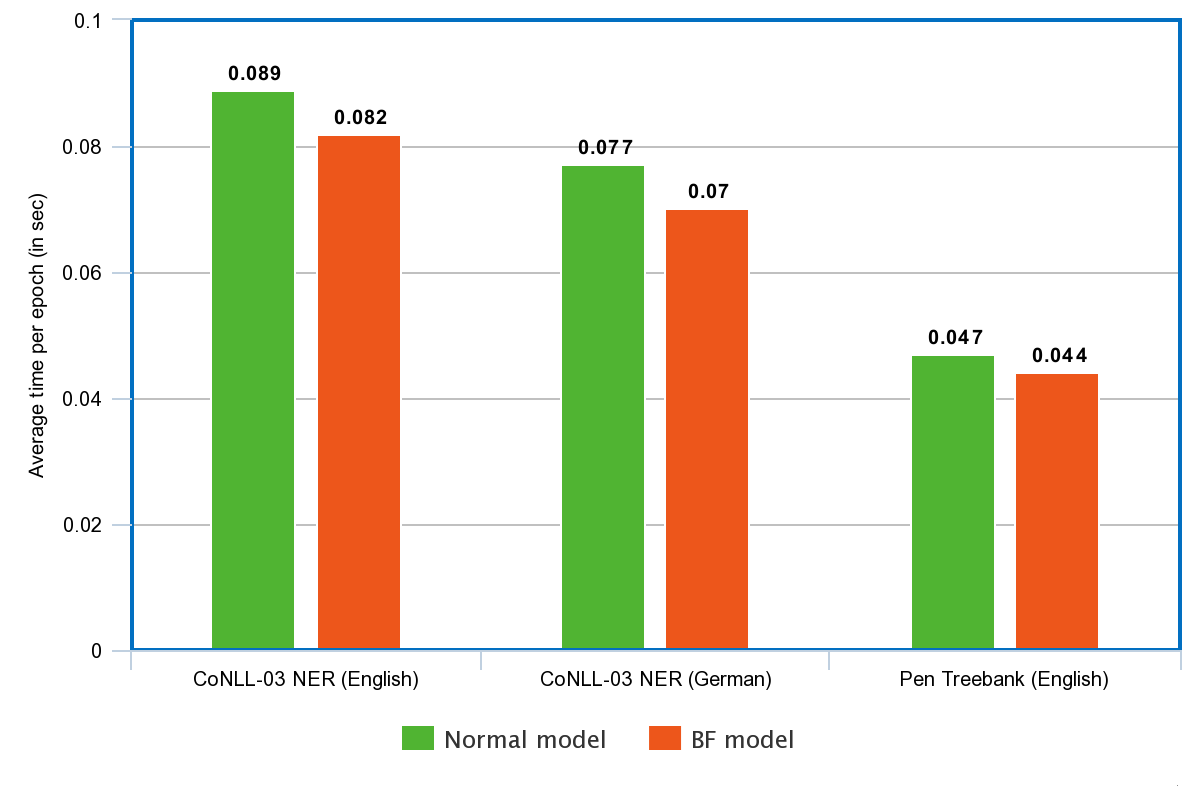}
\endminipage\hfill
\caption{Training/Inference times for NLP; Left: Training time, Right: Inference time}
  \label{figuretrainingtime for nlp}
\end{figure}

\begin{figure}[!htb]
\begin{center} 
\includegraphics[width=0.49\columnwidth]{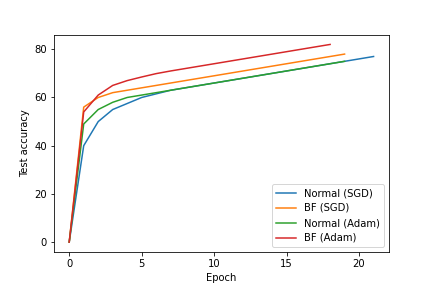}
 \end{center} 
\caption{Comparison of test accuracy in the first 20 epochs with different models and optimizers on CIFAR-10 with PreActResNet18}
  \label{figure20epochs}
\end{figure}

\subsection{Plots from Section \ref{EXP:AC}}\label{subsecappendix: plots from truncated butterfly}

Figure \ref{figapproxappendix} reports the losses for the Gaussian 2, Olivetti, and  Hyper data matrices.
\begin{figure}[!htb]
\minipage{0.32\textwidth}
  \includegraphics[width=\linewidth]{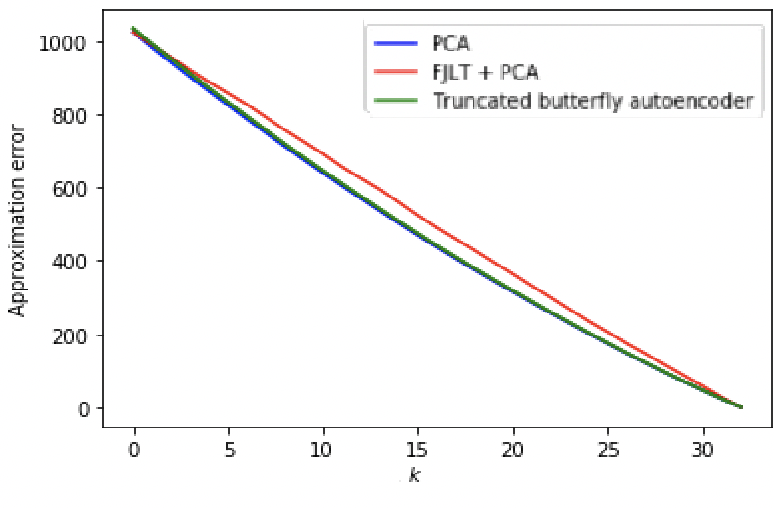}
  \label{rank32}
\endminipage\hfill
\minipage{0.32\textwidth}
  \includegraphics[width=\linewidth]{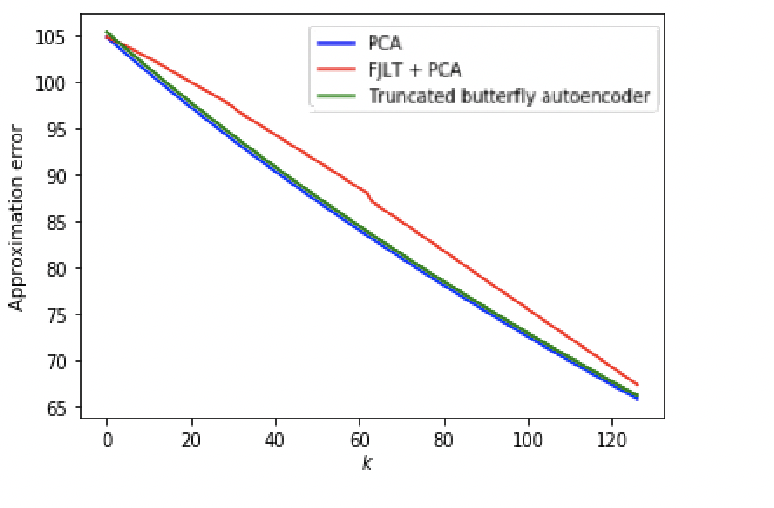}
  \label{faces}
\endminipage\hfill
\minipage{0.32\textwidth}
  \includegraphics[width=\linewidth]{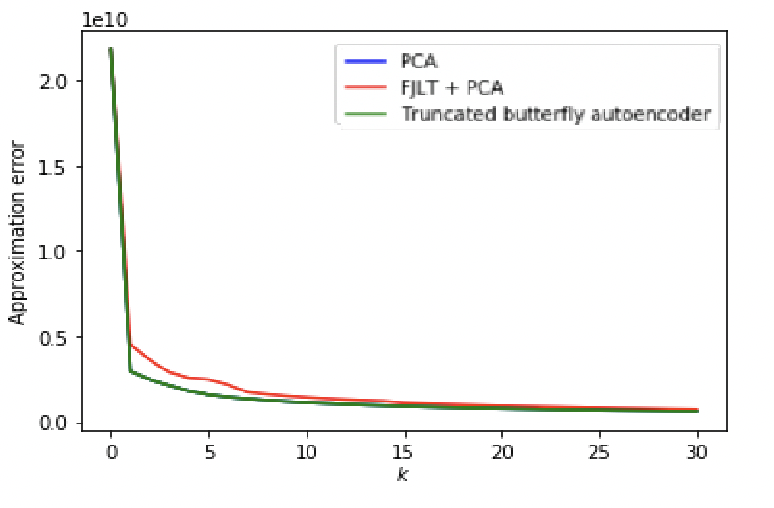}
  \label{hyper}
\endminipage\hfill
\caption{Approximation error on data matrix with various methods for various values of $k$. From left to right: Gaussian 2, Olivetti, Hyper}
\label{figapproxappendix}
\end{figure}
\section{Additional Plots related to Sketching}\label{subsecappendix: plots from supervise learned butterfly}

In this section we state a few additional cases that were done as part of the experiment in Section \ref{subsec: supervise learned butterfly}. Figure \ref{figtesterrl20k1} compares the test errors of the different methods in the extreme case when $k=1$. Figure \ref{figtesterrork10} compares the test errors of the different methods for various values of $\ell$. Figure \ref{figurelossperstep} shows the test error for $\ell=20$ and $k=10$ during the training phase on HS-SOD. Observe that the butterfly learned is able to surpass sparse learned after a merely few iterations. Finally Table \ref{figuretable} compares the test error for different values of $\ell$ and $k$.

\begin{figure}[!htb]
\minipage{0.45\textwidth}
  \includegraphics[width=\linewidth]{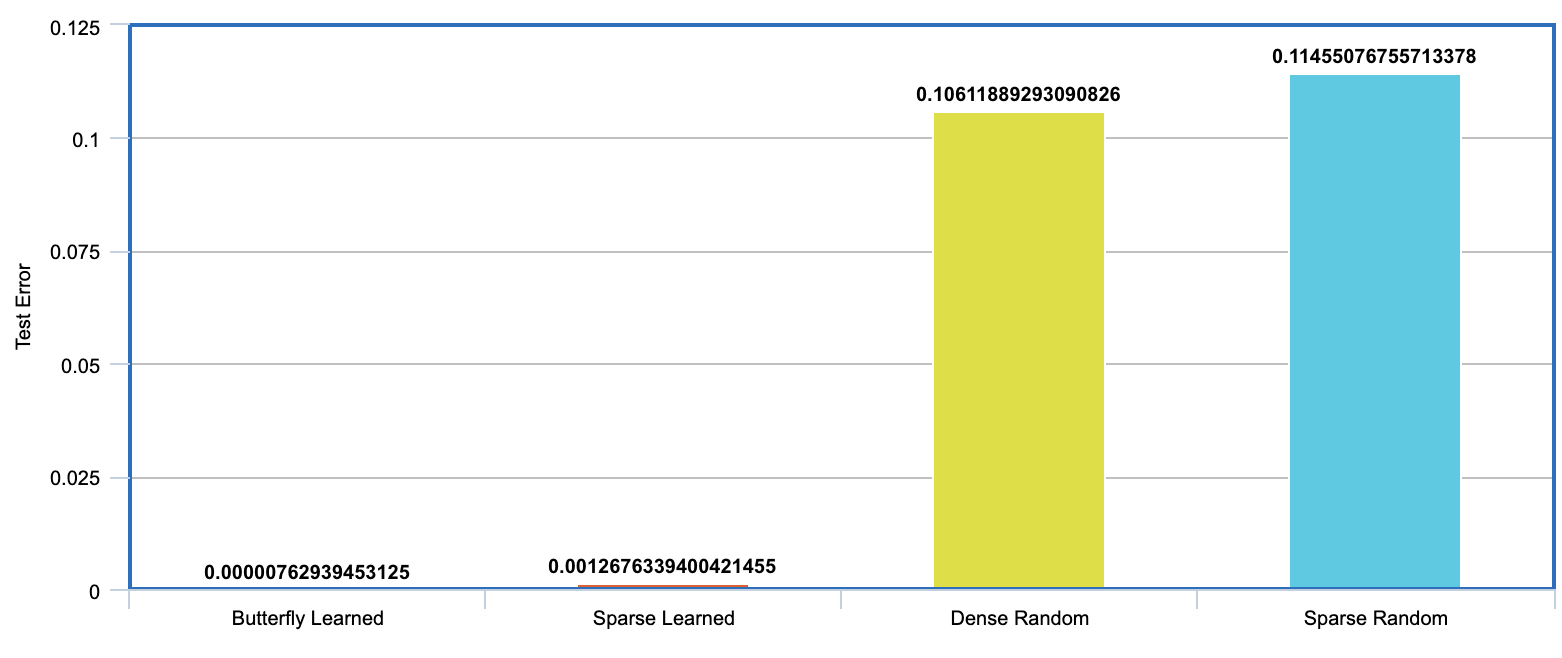}
\endminipage\hfill
\minipage{0.45\textwidth}%
  \includegraphics[width=\linewidth]{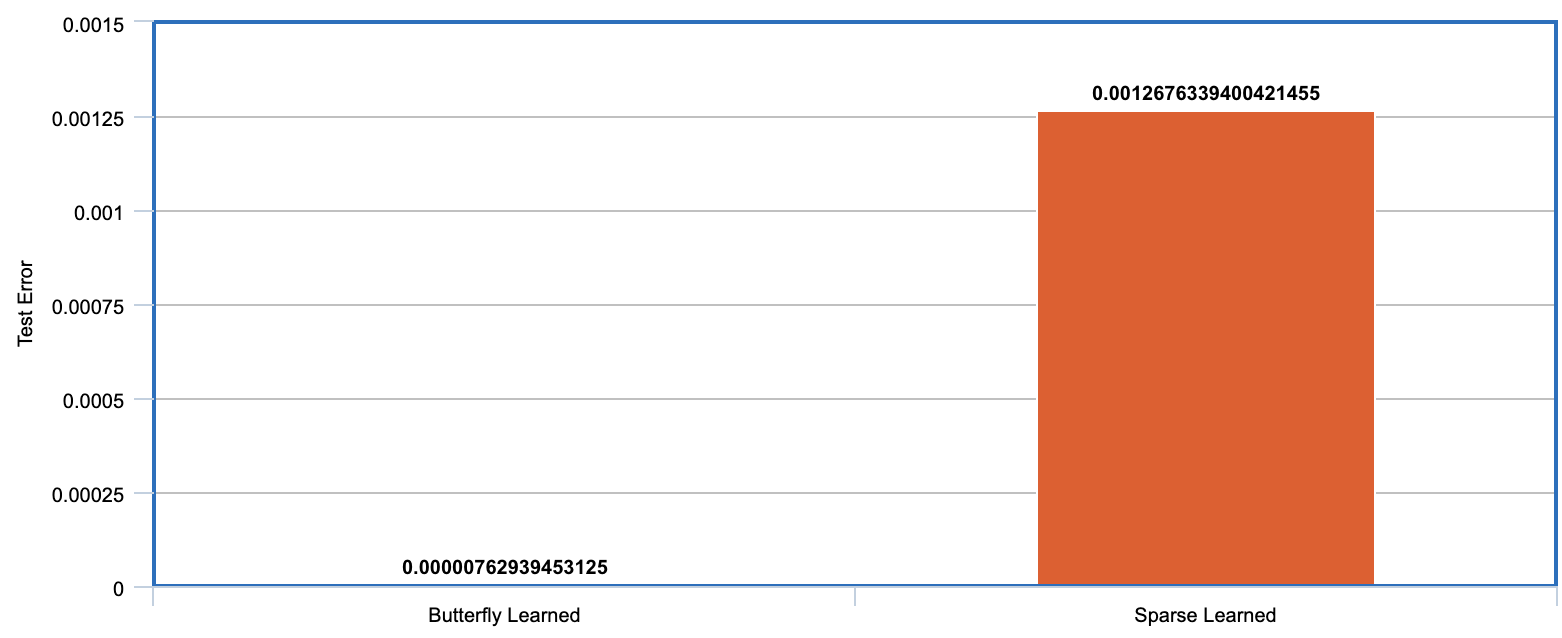}
\endminipage
\caption{Test errors on HS-SOD for $\ell=20$ and $k=1$, zoomed on butterfly and sparse learned in the right}
\label{figtesterrl20k1}
\end{figure}

\begin{figure}[!htb]
\minipage{0.45\textwidth}
  \includegraphics[width=\linewidth]{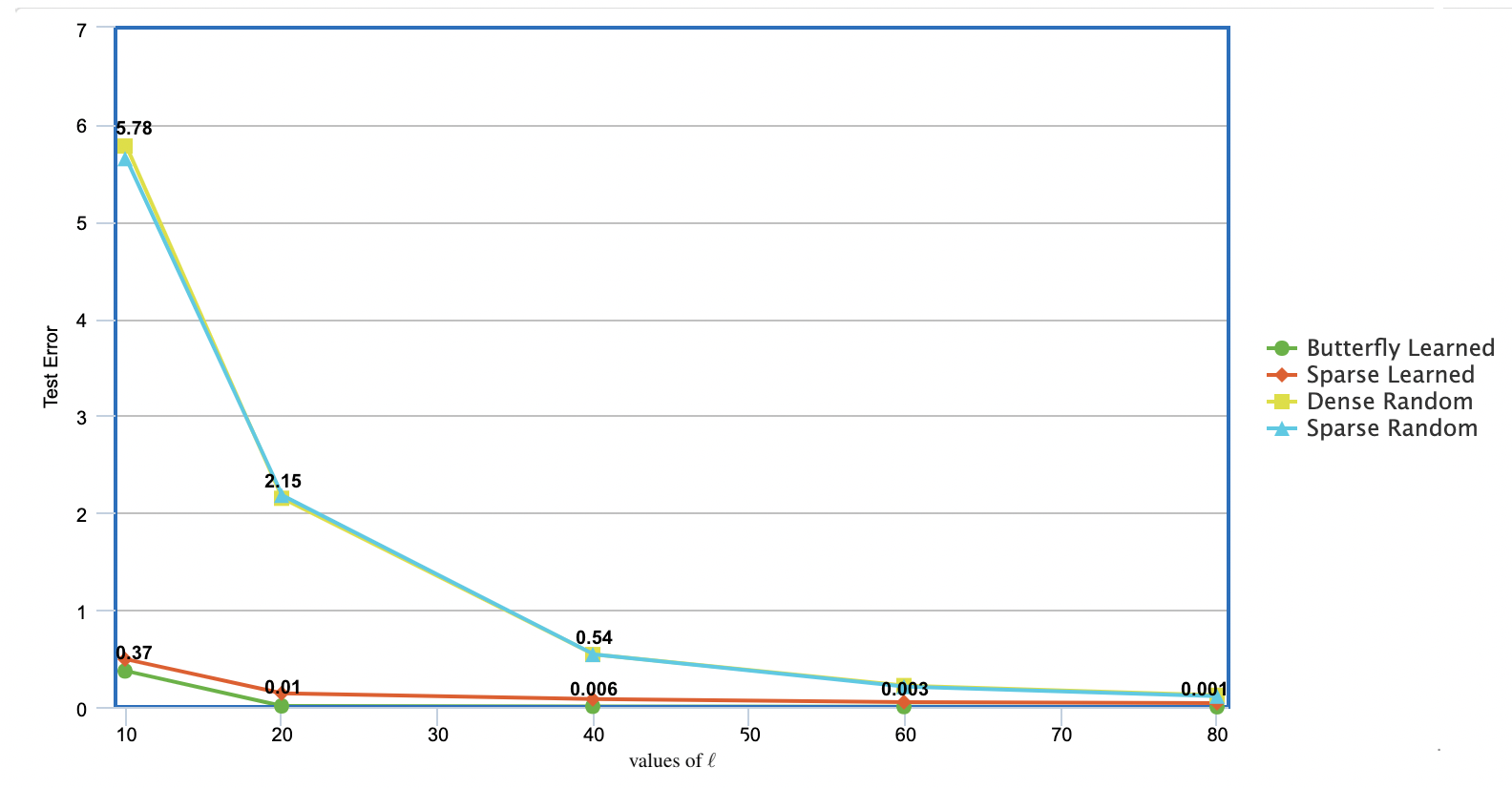}
  \label{figurek10}
\endminipage\hfill
\minipage{0.45\textwidth}%
  \includegraphics[width=\linewidth]{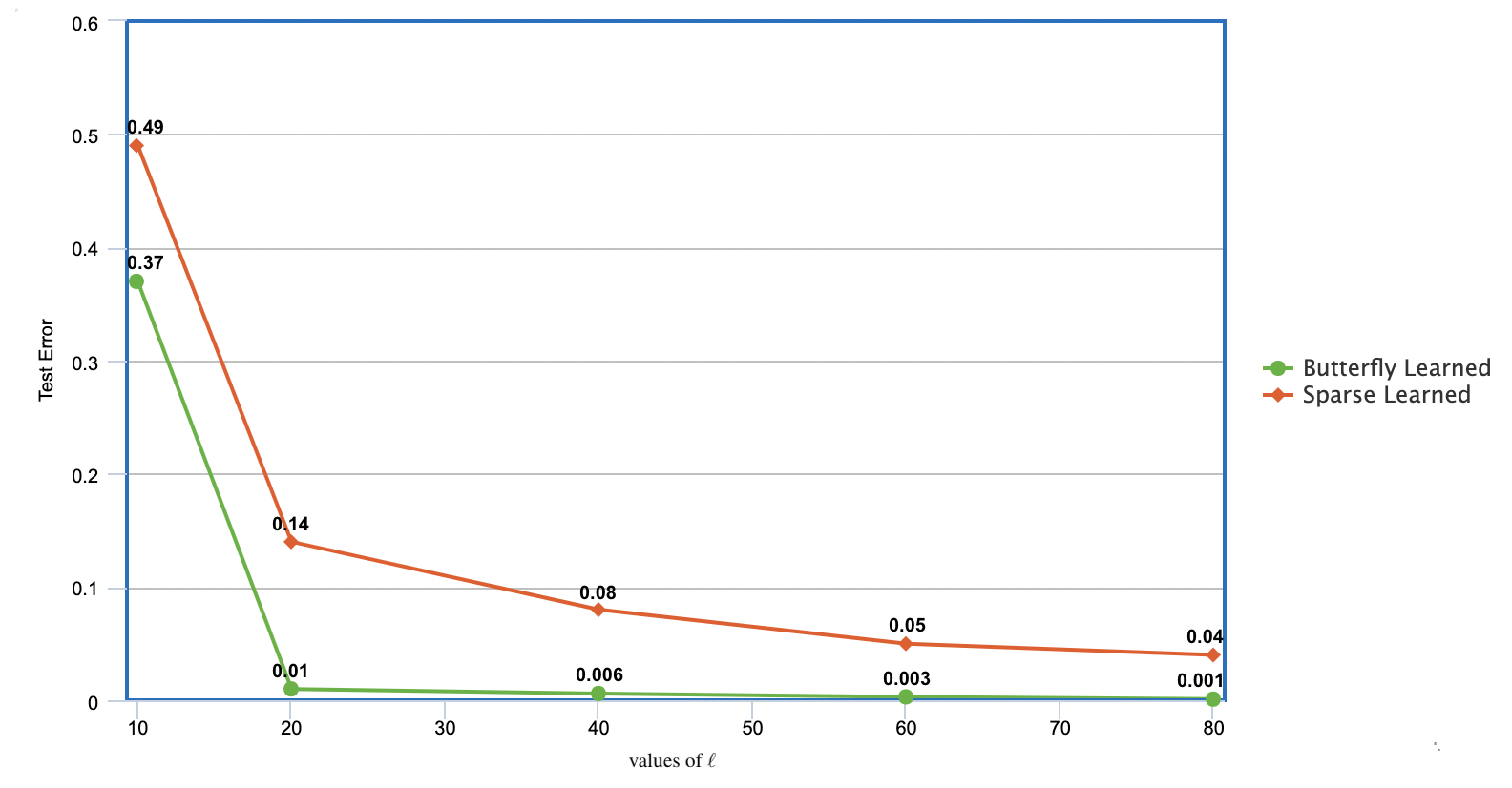}
  \label{figurezoom2}
\endminipage
\caption{Test error when $k=10$, $\ell= [10,20,40,60,80]$ on HS-SOD, zoomed on butterfly and sparse learned in the right}
\label{figtesterrork10}
\end{figure}

\begin{figure}[!htb]
\centering
  \includegraphics[width=0.49\columnwidth]{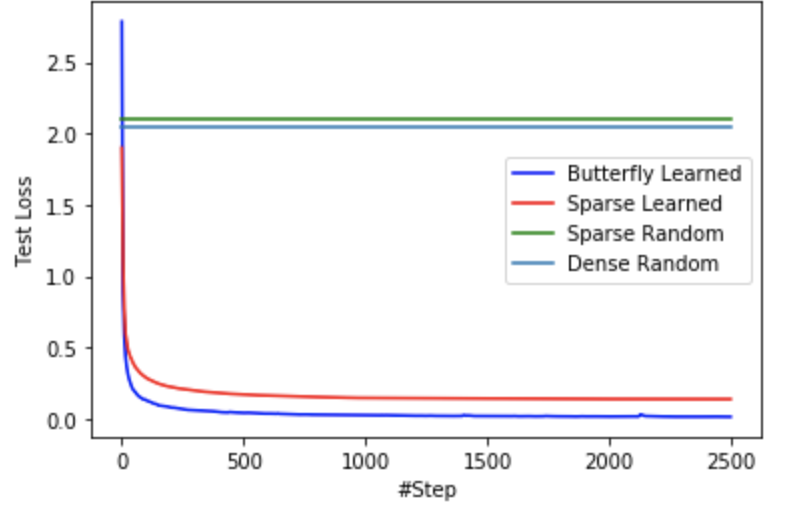}
  \caption{Test error when $k=10$, $\ell=20$ during the training phase on HS-SOD}
  \label{figurelossperstep}
\end{figure}

\begin{table}[!htb]
\centering
\begin{tabular}{|l|l|l|l|}
\hline
\textbf{k, $\ell$, Sketch}                                                                       & \textbf{Hyper}                                                 & \textbf{Cifar-10}                                              & \textbf{Tech}                                                  \\ \hline
\begin{tabular}[c]{@{}l@{}}~1, 5, Butterfly\\ ~1, 5, Sparse\\ ~1, 5, Random\end{tabular}       & \begin{tabular}[c]{@{}l@{}}\textbf{0.0008}\\ 0.003\\ 0.661\end{tabular} & \begin{tabular}[c]{@{}l@{}}\textbf{0.173}\\ 1.121\\ 4.870\end{tabular}  & \begin{tabular}[c]{@{}l@{}}\textbf{0.188}\\ 1.75\\ 3.127\end{tabular}   \\ \hline
\begin{tabular}[c]{@{}l@{}}~1, 10, Butterfly\\ ~1, 10, Sparse\\ ~1, 10, Random\end{tabular}    & \begin{tabular}[c]{@{}l@{}}\textbf{0.0002}\\ 0.002\\ 0.131\end{tabular} & \begin{tabular}[c]{@{}l@{}}\textbf{0.072}\\ 0.671\\ 1.82\end{tabular}   & \begin{tabular}[c]{@{}l@{}}\textbf{0.051}\\ 0.455\\ 1.44\end{tabular}   \\ \hline
\begin{tabular}[c]{@{}l@{}}10, 10, Butterfly\\ 10, 10, Sparse\\ 10, 10, Random\end{tabular} & \begin{tabular}[c]{@{}l@{}}\textbf{0.031}\\ 0.489\\ 5.712\end{tabular}  & \begin{tabular}[c]{@{}l@{}}\textbf{0.751}\\ 6.989\\ 26.133\end{tabular} & \begin{tabular}[c]{@{}l@{}}\textbf{0.619}\\ 7.154\\ 18.805\end{tabular} \\ \hline
\begin{tabular}[c]{@{}l@{}}10, 20, Butterfly\\ 10, 20, Sparse\\ 10, 20, Random\end{tabular} & \begin{tabular}[c]{@{}l@{}}\textbf{0.012}\\ 0.139\\ 2.097\end{tabular}  & \begin{tabular}[c]{@{}l@{}}\textbf{0.470}\\ 3.122\\ 9.216\end{tabular}  & \begin{tabular}[c]{@{}l@{}}\textbf{0.568}\\ 3.134\\ 8.22\end{tabular}   \\ \hline
\begin{tabular}[c]{@{}l@{}}10, 40, Butterfly\\ 10, 40, Sparse\\ 10, 40, Random\end{tabular} & \begin{tabular}[c]{@{}l@{}}\textbf{0.006}\\ 0.081\\ 0.544\end{tabular}  &                                                                & \begin{tabular}[c]{@{}l@{}}\textbf{0.111}\\ 0.991\\ 3.304\end{tabular}  \\ \hline
\begin{tabular}[c]{@{}l@{}}20, 20, Butterfly\\ 20, 20, Sparse\\ 20, 20, Random\end{tabular} & \begin{tabular}[c]{@{}l@{}}\textbf{0.058}\\ 0.229\\ 4.173\end{tabular}  &                                                                & \begin{tabular}[c]{@{}l@{}}\textbf{1.38}\\ 8.14\\ 15.268\end{tabular}   \\ \hline
\begin{tabular}[c]{@{}l@{}}20, 40, Butterfly\\ 20, 40, Sparse\\ 20, 40, Random\end{tabular} & \begin{tabular}[c]{@{}l@{}}\textbf{0.024}\\ 0.247\\ 1.334\end{tabular}  &                                                                & \begin{tabular}[c]{@{}l@{}}\textbf{0.703}\\ 3.441\\ 6.848\end{tabular}  \\ \hline
\begin{tabular}[c]{@{}l@{}}30, 30, Butterfly\\ 30, 30, Sparse\\ 30, 30, Random\end{tabular} & \begin{tabular}[c]{@{}l@{}}\textbf{0.027}\\ 0.749\\ 3.486\end{tabular}  &                                                                & \begin{tabular}[c]{@{}l@{}}\textbf{1.25}\\ 7.519\\ 13.168\end{tabular}  \\ \hline
\begin{tabular}[c]{@{}l@{}}30, 60, Butterfly\\ 30, 60, Sparse\\ 30, 60, Random\end{tabular} & \begin{tabular}[c]{@{}l@{}}\textbf{0.014}\\ 0.331\\ 2.105\end{tabular}  &                                                                & \begin{tabular}[c]{@{}l@{}}\textbf{0.409}\\ 2.993\\ 5.124
\end{tabular}  \\ \hline

\end{tabular}
\caption{Test error for different $\ell$ and $k$}
\label{figuretable}
\end{table}

\section{Bound on Number of Effective Weights in Truncated Butterfly Network}\label{appendix:proof:2nlog2ell}
A butterfly network for dimension $n$, which we assume for simplicity to be an integral power of $2$, is $\log n$ layers deep.  Let $p$ denote the integer $\log n$.  
The set of nodes in the first (input) layer will be denoted here by $V^{(0)}$.  They are connected to the set of $n$ nodes $V^{(1)}$ from the next layer, and so on until the nodes $V^{(p)}$ of the output layer. Between two consecutive layers $V^{(i)}$ and $V^{(i+1)}$, there are $2n$ weights, and each node in $V^{(i)}$ is adjacent to exactly two nodes from $V^{(i+1)}$.

When truncating the network, we discard  all but some set $S^{(p)}\subseteq V^{(p)}$ of at most $\ell$ nodes in the last layer.  These nodes are
connected to a subset $S^{(p-1)}\subseteq V^{(p-1)}$ of at most $2\ell$ nodes from the preceding layer using at most $2\ell$ weights.  By induction, for all $i\geq 0$, the set of nodes $S^{(p-i)}\subseteq V^{(p-i)}$ is of size at most $2^i\cdot\ell$, and is connected to the set
$S^{(p-i-1)}\subseteq V^{(p-i-1)}$ using at most $2^{i+1}\cdot\ell$ weights.

Now take $k=\lceil \log_2 (n/\ell) \rceil$.  By the above, the total number of weights that can  participate in a path connecting some node in $S^{(p)}$ with some node in $V^{(p-k)}$ is at most $2\ell+4\ell+\cdots + 2^k\ell \leq 4n$.  

From the other direction, the total number of weights that can participate in a path connecting any node from $V^{(0)}$ with any node from $V^{(p-k)}$ is $2n$ times the number of layers in between, or more precisely:
$$2n(p-k) = 2n(\log_2 n - \lceil \log_2 (n/\ell)\rceil) \leq 2n(\log_2 n - \log_2 (n/\ell) + 1) = 2n(\log \ell+1)\ .$$

The total is $2n\log \ell + 6n$, as required.
\end{document}